\documentclass{article}

\PassOptionsToPackage{numbers, sort, compress}{natbib}

\usepackage[preprint]{gsubsampling}
 
\usepackage[utf8]{inputenc} 
\usepackage[T1]{fontenc}    
\usepackage{url}            
\usepackage{booktabs}       
\usepackage{amsfonts}       
\usepackage{nicefrac}       
\usepackage{microtype}      

\usepackage{color}
\usepackage{xcolor}         
\usepackage[colorlinks]{hyperref}
\def\tmp#1#2#3{%
  \definecolor{Hy#1color}{#2}{#3}%
  \hypersetup{#1color=Hy#1color}}
\tmp{link}{HTML}{800006}
\tmp{cite}{HTML}{2E7E2A}
\tmp{file}{HTML}{131877}
\tmp{url} {HTML}{8A0087}
\tmp{menu}{HTML}{727500}
\tmp{run} {HTML}{137776}
\def\tmp#1#2{%
  \colorlet{Hy#1bordercolor}{Hy#1color#2}%
  \hypersetup{#1bordercolor=Hy#1bordercolor}}
\tmp{link}{!60!white}
\tmp{cite}{!60!white}
\tmp{file}{!60!white}
\tmp{url} {!60!white}
\tmp{menu}{!60!white}
\tmp{run} {!60!white}

\usepackage{amsthm}
\usepackage{amsmath}
\usepackage{amssymb}
\usepackage{placeins}
\usepackage{threeparttable}
\usepackage{bm}
\usepackage{soul}
\usepackage[normalem]{ulem}
\usepackage{pifont}
\usepackage{cleveref}
\usepackage{paralist}
\usepackage{makecell}
\usepackage{caption}
\usepackage{multirow}
\usepackage{wrapfig}
\usepackage{caption}
\usepackage{subcaption}
\usepackage{thmtools, thm-restate}
\usepackage{bbold}
\usepackage{dsfont}

\usepackage[inline]{enumitem}
\usepackage{mathtools}

\theoremstyle{definition}

\title{Group Equivariant Subsampling}

\author{%
  \bf{Jin Xu}$^{1}$\thanks{Corresponding author: \texttt{<jin.xu@stats.ox.ac.uk>}} \hspace{0.6cm} \bf{Hyunjik Kim}$^{2}$ \hspace{0.6cm} \bf{Tom Rainforth}$^{1}$ \hspace{0.6cm}
  \bf{Yee Whye Teh}$^{1,2}$ \\ \\
  $^1$ Department of Statistics, University of Oxford, UK.\\
  $^2$ DeepMind, UK. 
}

\begin{document}

\maketitle

\begin{abstract}

Subsampling is used in convolutional neural networks (CNNs) in the form of pooling or strided convolutions, to reduce the spatial dimensions of feature maps and to allow the receptive fields to grow exponentially with depth.
However, it is known that such subsampling operations are not translation equivariant, unlike convolutions that \emph{are} translation equivariant. 
Here, we first introduce translation equivariant subsampling/upsampling layers that can be used to construct exact translation equivariant CNNs.
We then generalise these layers beyond translations to general groups, thus proposing \textit{group equivariant subsampling/upsampling}.
We use these layers to construct group equivariant autoencoders (GAEs) that allow us to learn low-dimensional equivariant representations.
We empirically verify on images that the representations are indeed equivariant to input translations and rotations, and thus generalise well to unseen positions and orientations. We further use GAEs in models that learn object-centric representations on multi-object datasets, and show improved data efficiency and decomposition compared to non-equivariant baselines.
\end{abstract}

\section{Introduction} \label{sec:introduction}

\looseness=-1
Convolutional Neural Networks (CNNs) are known to be more data efficient and show better generalisation on perceptual tasks than fully-connected networks, due to translation equivariance encoded in the convolutions: when the input image/feature map is translated, the output feature map also translates by the same amount. 
In typical CNNs, convolutions are used in conjunction with subsampling operations, in the form of pooling or strided convolutions, to reduce the spatial dimensions of feature maps and to allow receptive field to grow exponentially with depth. 
Subsampling/upsampling operations are especially necessary for convolutional autoencoders (ConvAEs) \citep{Masci2011StackedCA} because they allow efficient dimensionality reduction.
However, it is known that subsampling operations implicit in strided convolutions or pooling layers are \textit{not} translation equivariant \citep{zhang2019making}, hence CNNs that use these components are also not translation invariant.
Therefore such CNNs and ConvAEs are not guaranteed to generalise to arbitrarily translated inputs despite their convolutional layers being translation equivariant.

Previous work, such as \cite{zhang2019making,chaman2020truly}, has investigated how to enforce translation invariance on CNNs, but does not study equivariance with respect to symmetries beyond translations, such as rotations or reflections.
In this work, we first describe subsampling/upsampling operations that preserve exact translation equivariance. 
The main idea is to sample feature maps on an input-dependent grid rather than a fixed one as in pooling or strided convolutions, and the grid is chosen according to a \textit{sampling index} computed from the inputs (see \Cref{fig:subsampling}).
Simply replacing the subsampling/upsampling in standard CNNs with such translation equivariant subsampling/upsampling operations leads to CNNs and transposed CNNs that can map between spatial inputs and low-dimensional representations in a translation equivariant manner.

\looseness=-1
We further generalise the proposed subsampling/upsampling operations from translations to arbitrary groups, proposing \textit{group equivariant subsampling/upsampling}.
In particular we identify subsampling as mapping features on groups $G$ to features on subgroups $K$ (vice versa for upsampling), and identify the sampling index as a coset in the quotient space $G/K$. See \Cref{sec:app_preliminaries} for a primer on group theory that is needed to describe this generalisation.
We note that group equivariant subsampling is different to \textit{coset pooling} introduced in \cite{cohen2016group}, which instead gives features on the quotient space $G/K$, and discuss differences in detail in \Cref{sec:related_work}.
Similar to the translation equivariant subsampling/upsampling, group equivariant subsampling/upsampling can be used with group equivariant convolutions to produce group equivariant CNNs.
Using such group equvariant CNNs we can construct group equivariant autoencoders (GAEs) that separate representations into an invariant part and an equivariant part.

\looseness=-1
While there is a growing body of literature on group equivariant CNNs (G-CNNs) \citep{cohen2016group,cohen2016steerable,worrall2017harmonic,weiler2018learning,weiler20183d,thomas2018tensor,weiler2019general}, such equivariant convolutions usually preserve the spatial dimensions of the inputs (or lift them to even higher dimensions) until the final invariant pooling layer. There is a lack of exploration on how to reduce the spatial dimensions of such feature maps while preserving exact equivariance, to produce low-dimensional equivariant representations.
This work attempts to fill in this gap.
Such low-dimensional equivariant representations can be employed in representation learning methods, allowing various advantages such as interpretability, out-of-distribution generalisation, and better sample complexity. 
When using such learned representations in downstream tasks such as abstract reasoning, reinforcement learning, video modelling, scene understanding, it is especially important for representations to be equivariant rather than invariant in these tasks, because transformations and how they act on feature spaces are critical information, rather than nuisance as in image classification problems.

\begin{figure}[t]
  \centering
    \includegraphics[width=0.99\linewidth]{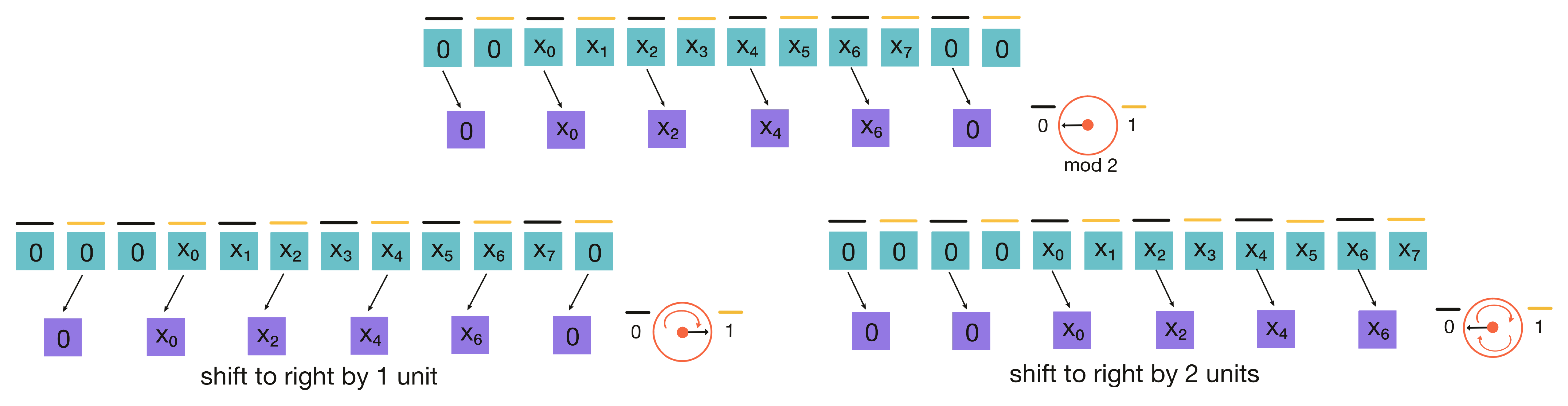}
    \caption{Equivariant subsampling on 1D feature maps with a scale factor $c=2$. The input feature map has length $8$, and initially we sample from odd positions determined by \Cref{eq:argmax} (top). When the original feature map is shifted to the right by $1$ unit (bottom left), the sampling index becomes $1$, so we instead sample from even positions. When the feature map is shifted to the right by $2$ units (bottom right), we again sample from odd positions, but the outputs have been shifted to the right by $1$ unit correspondingly.}
    \label{fig:subsampling}
\end{figure}

In summary, we make the following contributions:
\begin{enumerate*}[label=(\roman*)]
    \item We propose subsampling/upsampling operations that preserve translational equivariance.
    \item We generalise the proposed subsampling/upsampling operations to arbitrary symmetry groups.
    \item We use equivariant subsampling/upsampling operations to construct GAEs that gives low-dimensional equivariant representations.
    \item We empirically show that representations learned by GAEs enjoys many advantages such as interpretability, out-of-distribution generalisation, and better sample complexity.
\end{enumerate*}

\section{Equivariant Subsampling and Upsampling} \label{sec:equivariant_subsampling}

\subsection{Translation Equivariant Subsampling for CNNs} \label{subsec:translation_equivariant_subsampling}

In this section we describe the proposed translation equivariant subsampling scheme for feature maps in standard CNNs. Later in \Cref{subsec:group_equivariant_subsampling}, we describe how this can be generalised to group equivariant subsampling for feature maps on arbitrary groups.

\paragraph{Standard subsampling} Feature maps in CNNs can be seen as functions defined on the integer grid, e.g. $\mathbb{Z}$ for 1D feature maps, and $\mathbb{Z}^2$ for 2D. Hence we represent feature maps as $f: \mathbb{Z} \rightarrow R^d$, where $d$ is the number of feature map channels. For simplicity, we start with 1D and move on to the 2D case.
Typically, subsampling in CNNs is implemented as either strided convolution or (max) pooling, and they can be decomposed as
\begin{align*}
    \operatorname{\textsc{conv}}_k^c &= \operatorname{\textsc{subsampling}}^c \circ \operatorname{\textsc{conv}}_k^1 \nonumber \\
    \operatorname{\textsc{maxpool}}_k^c &= \operatorname{\textsc{subsampling}}^c \circ \operatorname{\textsc{maxpool}}_k^1
\end{align*}
where subscripts denote kernel sizes and superscripts indicate strides. $c \in \mathbb{N}$ is the scale factor for $\operatorname{\textsc{subsampling}}$, and this operation simply restricts the input domain of the feature map from $\mathbb{Z}$ to $c\mathbb{Z}$, without changing the corresponding function values.

\paragraph{Translation equivariant subsampling} In our equivariant subsampling scheme, we instead restrict the input domain to $c\mathbb{Z}+i$, the integers $\equiv i \mod c$, where $i$ is a sampling index determined by the input feature map.
The key idea is to choose $i$ such that it is shifts by $t (\bmod c)$ when the input is translated by $t$,
to ensure that the same features are subsampled upon translation.
Let $i$ be given by the mapping $\Phi_c:\mathcal{I}_{\mathbb{Z}}\rightarrow \mathbb{Z}/c\mathbb{Z}$. $\mathcal{I}_{\mathbb{Z}}$ denotes the space of vector functions on $\mathbb{Z}$ and $\mathbb{Z}/c\mathbb{Z}$ is the space of remainders upon division by $c$.
\begin{align}
    i = \Phi_c(f) = \bmod(\arg\max_{x\in \mathbb{Z}} \|f(x)\|_1, c) \label{eq:argmax}
\end{align}
where $\|\cdot \|_1$ denotes $L^1$-norm (other choices of norm are equally valid). 
Other choices for $\Phi_c$ are equally valid as long as they satisfy translation equivariance, ensuring that the same features are subsampled upon translation of the input:
\begin{align}
    \Phi_c(f(\cdot-t)) = \bmod(\Phi_c(f)+t, c).
\end{align}
Note that this holds for \Cref{eq:argmax} provided the argmax is unique, which we assume for now (see \Cref{sec:app_phi} for a discussion of the non-unique case). 
We can decompose the subsampled feature map defined on $c\mathbb{Z}+i$ into its values and the offset index $i$, expressing it as $[f_b, i]\in (\mathcal{I}_{c\mathbb{Z}}, \mathbb{Z}/c\mathbb{Z})$, where $f_b$ is the translated output feature map such that $f_b(cx)=f(cx+i)$ for $x\in \mathbb{Z}$.

\looseness=-1
The subsampling operation described above, which maps from $\mathcal{I}_{\mathbb{Z}}$ to $(\mathcal{I}_{c\mathbb{Z}}, \mathbb{Z}/c\mathbb{Z})$ is translation equivariant: when the feature map $f$ is translated to the right by $t\in\mathbb{Z}$, one can verify that $f_b$ will be translated to the right by $c \lfloor \frac{i+t}{c} \rfloor$, and the sampling index for the translated inputs would become $\bmod(i+t,c)$. We provide an illustration for $c=2$ in \Cref{fig:subsampling}, and describe formal statements and proofs later for the general cases in \Cref{subsec:group_equivariant_subsampling}.

\paragraph{Multi-layer case} For the subsequent layers, the feature map $f_b$ is fed into the next convolution, and the sampling index $i$ is appended to a list of outputs. When the above translation equivariant subsampling scheme is interleaved with convolutions in this way, we obtain an exactly translation equivariant CNN, where each subsampling layer with scale factor $c_k$ produces a sampling index $i_k \in \mathbb{Z}/c_k \mathbb{Z}$. Hence the equivariant representation output by the CNN with $L$ subsampling layers is a final feature map $f_L$ and a $L$-tuple of sampling indices $(i_1, \ldots, i_L)$. This tuple can in fact be expressed equivalently as a single integer by treating the tuple as mixed radix notation and converting to decimal notation. We provide details of this multi-layer case in \Cref{sec:app_multilayer}, including a rigorous formulation and its equivariance properties.

\paragraph{Translation equivariant upsampling} As a counterpart to subsampling, upsampling operations increase the spatial dimensions of feature maps. We propose an equivariant upsampling operation that takes in a feature map $f \in \mathcal{I}_{c\mathbb{Z}}$ and a sampling index $i \in \mathbb{Z}/c\mathbb{Z}$, and outputs a feature map $f_u \in \mathcal{I}_{\mathbb{Z}}$, where we set $f_u(cx+i)=f(cx)$ and $\bm{0}$ everywhere else. This works well enough in practice, although in conventional upsampling the output feature map is often a smooth interpolation of the input feature map. To achieve this with equivariant upsampling, we can additionally apply average pooling with stride $1$ and kernel size $> 1$. 

\paragraph{2D Translation equivariant subsampling} When feature maps are $2D$, they can be represented as functions on $\mathbb{Z}^2$. The sampling index becomes a $2$-element tuple given by:
\begin{align*}
    (x^{\ast},y^{\ast}) &= {\arg\max}_{(x,y)\in \mathbb{Z}^2} \|f(x)\|_1 \nonumber \\ 
    (i, j) &= (\bmod(x^{\ast}, c), \bmod(y^{\ast}, c))
\end{align*}
and we subsample feature maps by restricting the input domain to $c\mathbb{Z}^2+(i,j)$. The multi-layer construction and upsampling is analogous to the 1D-case.

\subsection{Group Equivariant Subsampling and Upsampling} \label{subsec:group_equivariant_subsampling}
In this section, we propose group equivariant subsampling by starting off with the 1D-translation case in \Cref{subsec:translation_equivariant_subsampling}, and provide intuition for how each component of this special case generalises to arbitrary discrete groups $G$. We then proceed to mathematically formulate group equivariant subsampling, and prove that it is indeed $G$-equivariant.

\paragraph{Feature maps on groups} First recall that the feature maps for the 1D-translation case were defined as functions on $\mathbb{Z}$, or $f \in \mathcal{I}_{\mathbb{Z}}$ for short. To extend this to the general case, we consider feature maps $f$ as functions on a group $G$, i.e. $ f \in \mathcal{I}_G=\{f:G\to V\}$\footnote{This is not to be confused with the space of Mackey functions in, e.g., \cite{cohen2019general}, and rather it is the space of unconstrained functions on $G$.} where $V$ is a vector space, as is done in e.g. group equivariant CNNs (G-CNNs) \citep{cohen2016group}.
Note that translating feature maps $f$ on $\mathbb{Z}$ by displacement $u$ is effectively defining a new feature map $f'(\cdot) = f(\cdot - u)$. In the general case, we say that the group action on the feature space is given by
\begin{align} \label{eq:action_pi}
    [\pi(u)f](g) = f(u^{-1} g)
\end{align}
where $\pi$ is a group representation describing how $u\in G$ acts on the feature space.

\paragraph{Recap: translation equivariant subsampling} Recall that standard subsampling that occurs in pooling or strided convolutions for 1D translations amounts to restricting the domain of the feature map from $\mathbb{Z}$ to $c\mathbb{Z}$, whereas equivariant subsampling also produces a sampling index $i \in \mathbb{Z}/c\mathbb{Z}$, an integer mod $c$, and that this is equivalent to restricting the input domain to $c\mathbb{Z} + i$. $i$ is given by the translation equivariant mapping $\Phi_c: \mathcal{I}_{\mathbb{Z}} \rightarrow \mathbb{Z}/ c \mathbb{Z}$. We can translate the input domain back to $c\mathbb{Z}$, and represent the output of subsampling as $[f_b, i]\in (\mathcal{I}_{c\mathbb{Z}}, \mathbb{Z}/c\mathbb{Z})$, where $f_b$ is the translated output feature map and $f_b(cx)=f(cx+i)$ for $x\in \mathbb{Z}$. 

\paragraph{Group equivariant subsampling} Similarly in the general case, for a feature map $f\in \mathcal{I}_G$, standard subsampling can be seen as restricting the domain from the group $G$ to a subgroup $K$, whereas equivariant subsampling additionally produces a sampling index $pK \in G/K$, where the quotient space $G/K = \{gK: g \in G\}$ is the set of (left) \textit{cosets} of $K$ in $G$. Note that we have rewritten $i$ as $p$ to distinguish between the 1D translation case and the general group case. This is equivalent to restricting the $f$ to the coset $pK$.
The choice of the coset $pK$ is given by equivariant map $\Phi: \mathcal{I}_G \rightarrow G/K$ (the action of $G$ on $G/K$ is given by $u(gK)=(ug)K$ for $u,g\in G$), such that $pK = \Phi(f)$. This restriction of $f$ to $pK$ can also be thought of as having an output feature map $f_b$ on $K$ and choosing a coset representative element $\bar{p} \in pK$, such that $f_b(k) = f(\bar{p}k)$. This choice of coset representative is described by a function $s:G/K\to G$, such that $\bar{p}=s(pK)$. The function $s$ is called a section and should satisfy $s(pK)K=pK$.

Now let us formulate subsampling and upsampling operations $S_b{\downarrow}_K^G$ and $S_u{\uparrow}_K^G$ mathematically and prove its $G$-equivariance. Let $\mathcal{I}_K=\{f:K\to V'\}$ be the space of feature map on $K$. $S_b{\downarrow}_K^G$ takes in a feature map $f\in \mathcal{I}_G$ and produces a feature map $f_b\in \mathcal{I}_K$ and a coset in $G/K$. In reverse, the upsampling operation $S_u{\uparrow}_K^G$ takes in a feature map in $\mathcal{I}_K$, a coset in $G/K$, and produces a feature map in $\mathcal{I}_G$.  We use a section $s: G/K \rightarrow G$ to represent a coset with a representative element in $G$, and point out that equivariance holds for any choice of $s$.

Formally, given an equivariant map $\Phi:\mathcal{I}_G\rightarrow G/K$ (we will discuss how to construct such a map in \Cref{subsec:phi}), and a fixed section $s:G/K\rightarrow G$ such that $\bar{p} = s(p K)$, the subsampling operation $S_b{\downarrow}_K^G: \mathcal{I}_G \rightarrow \mathcal{I}_K \times G/K$ is defined as:
\begin{align} \label{eq:subsampling}
    p K &= \Phi(f), \hspace{3mm}
    f_b(k) = f(\bar{p} k) \;\text{for}\; k\in K \nonumber \\
    [f_b,  p K] &= S_b{\downarrow}_K^G(f;\Phi),
\end{align}
while the upsampling operation $S_u{\uparrow}_K^G: \mathcal{I}_K \times G/K \rightarrow \mathcal{I}_G$ is defined as:
\begin{align} \label{eq:upsampling}
    f_u(g) &= f(\bar{p}^{-1} g) \text{\;if\;} g\in K \text{\;else\;} \bm{0} \nonumber \\
    f_u &= S_u{\uparrow}_K^G(f, p K).
\end{align}

To make the output of the upsampling dense rather than sparse, one can apply arbitrary equivariant smoothing functions such as average pooling with stride $1$ and kernel size $> 1$, to compensate for the fact that we extend with $\bm{0}$s rather than values close to their neighbours. In practice, we observe that upsampling without any smoothing function works well enough.

The statement on the equivariance of $S_b{\downarrow}_K^G$ and $S_u{\uparrow}_K^G$ requires we specify the action of $G$ on the space $\mathcal{I}_K \times G/K$, which we denote as $\pi'$. For any $u\in G$,
\begin{align} \label{eq:group_action}
    p' K &= up K, \hspace{3mm}
    f'_b =  \pi(\bar{p}'^{-1}u\bar{p}) f_b \nonumber \\
    [f'_b,\;p' K] &= \pi'(u) [f_b,\;p K] 
\end{align}

\begin{restatable}{lemma}{gaction}
\label{thm:gaction}
$\pi'$ defines a valid group action of $G$ on the space $\mathcal{I}_K \times G/K$.
\end{restatable}
We can now state the following equivariance property (See \Cref{sec:app_proofs} for a proof):
\begin{restatable}{proposition}{equivariance}
\label{thm:equivariance}
If the action of group $G$ on the space $\mathcal{I}_G$ and $\mathcal{I}_K \times G/K$ are specified by $\pi,\pi'$ (as defined in \Cref{eq:action_pi,eq:group_action}), and $\Phi:\mathcal{I}_G\rightarrow G/K$ is an equivariant map, then the operations $S_b{\downarrow}_K^G$ and $S_u{\uparrow}_K^G$ as defined in \Cref{eq:subsampling,eq:upsampling} are equivariant maps between $\mathcal{I}_G$ and $\mathcal{I}_K \times G/K$.
\end{restatable}

\subsection{Constructing $\Phi$} \label{subsec:phi}

We use the following simple construction of the equivariant mapping $\Phi:\mathcal{I}_G \to G/K$ for subsampling/upsampling operations, although any equivariant mapping would suffice. 
For an input feature map $f\in \mathcal{I}_G$, we define
\begin{align} \label{eq:phi}
    pK = \Phi(f) \coloneqq (\arg\max_{g\in G} \|f(g)\|_1)K
\end{align}
Provided that the argmax is unique, it is easy to show that $(up)\cdot K = \Phi(\pi(u)f)$, hence $\Phi$ is equivariant. In practice one can insert arbitrary equivariant layers to $f$ before and after we take the norm $\|\cdot\|_1$ to avoid a non-unique argmax (see \Cref{sec:app_implementation}).

In theory, there could exist cases where the argmax is always non-unique. We provide a more complex construction of $\Phi$ that deals with this case in \Cref{sec:app_phi}.

\section{Application: Group Equivariant Autoencoders} \label{sec:group_equivariant_autoencoders}

Group equivariant autoencoders (GAEs) are composed of alternating G-convolutional layers and equivariant subsampling/upsampling operations for the encoder/decoder.
One important property of GAEs is that the final subsampling layer of the encoder subsamples to a feature map defined on the trivial group $\{e\}$, outputting a vector (instead of a feature map) that is \textit{invariant}.
For the 1D-translation case, suppose the input to the final subsampling layer is a feature map $f$ defined on $\mathbb{Z}$. Then the final layer produces an invariant vector $f_b(0)=f(i_L)$ where $i_L=\arg\max_{x\in \mathbb{Z}} \|f(x)\|_1$. Note that there is no scale factor $c_L$ here. Intuitively we can think of this as setting the scale factor $c_L=\infty$.
Hence the encoder of the GAE outputs a representation that is disentangled into an invariant part $z_{\text{inv}} = f_b(0)$ (the vector output by the final subsampling layer) and an equivariant part $z_{\text{eq}}=(i_1,...,i_L)$.

For the general group case, instead of specifying scale factors as in \Cref{subsec:translation_equivariant_subsampling}, we specify a sequence of nested subgroups $G=G_0 \geq G_1 \geq \dots \geq G_L=\{e\}$, where the feature map for layer $l$ is defined on subgroup $G_L$. For example, for the $p4$ group $G=\mathbb{Z}\rtimes \mathsf{C}_4$, we can use the following sequence for subsampling: $\mathbb{Z}\rtimes \mathsf{C}_4 \geq 2\mathbb{Z}\rtimes \mathsf{C}_4 \geq 4\mathbb{Z}\rtimes \mathsf{C}_4 \geq 8\mathbb{Z}\rtimes \mathsf{C}_2 \geq \{e\}$. Note that for the final two layers of this example, we are subsampling translations and rotations jointly.

We lift the input defined on the homogeneous input space to $\mathcal{I}_G$ (see Appendix \ref{sec:app_lifting} for details on homogeneous spaces and lifting), and treat $f_0 \in \mathcal{I}_G$ as inputs to the autoencoders. The group equivariant encoder $\operatorname{\textsc{enc}}$ can be described as follows:
\begin{align} \label{eq:encoder}
    [f_{l}, \;p_l G_l] &= S_b{\downarrow}_{G_{l}}^{G_{l-1}}( \operatorname{\textsc{g-cnn}}_{l-1}^E(f_{l-1}); \Phi_l) \nonumber \\
    [z_{\text{inv}}, z_{\text{eq}}] &= [f_L(e),\; (p_1G_1, p_2G_2, \ldots ,p_LG_L)]
\end{align}
where $l=1,\dots,L$ and $\operatorname{\textsc{g-cnn}}_l(\cdot)$ denotes G-convolutional layers before the $l$th subsampling layer.

The decoder $\operatorname{\textsc{dec}}$ simply goes in the opposite direction, and can be written formally as:
\begin{align} \label{eq:decoder}
    f_L &\text{ is defined on }G_L=\{e\} \text{ and } f_L(e) = z_{\text{inv}} \nonumber \\ 
    f_{l-1} &= \operatorname{\textsc{g-cnn}}_{l-1}^D(S_u{\uparrow}_{G_{l}}^{G_{l-1}}(f_l, \;p_l G_l))
\end{align}
where $l=L,\dots,1$ and $\hat{f}=f_0$ gives the final reconstruction.

Recall from \Cref{subsec:translation_equivariant_subsampling} that the tuple $(i_1,\ldots,i_L)$ can be expressed equivalently as a single integer. Similarly, the tuple $(p_1G_1, p_2G_2, \ldots ,p_LG_L)$ can be expressed as a single group element in $G$. We show in \Cref{sec:app_multilayer} that the action implicitly defined on the tuple via \Cref{eq:group_action} simplifies elegantly to the left-action on the single group element in $G$.

We now have the following properties for the learned representations (see \Cref{sec:app_proofs} for a proof):
\begin{restatable}[]{proposition}{inveqvrep}
\label{thm:inv-eqv-rep}
When $\operatorname{\textsc{enc}}$ and $\operatorname{\textsc{dec}}$ are given by \Cref{eq:encoder,eq:decoder}, and the group actions are specified as in \Cref{eq:action_pi} and \Cref{eq:group_action}, for any $g\in G$ and $f\in \mathcal{I}_G$, we have
\begin{align*} 
    [z_{\text{inv}}, g\cdot z_{\text{eq}}] &= \operatorname{\textsc{enc}}(\pi(g) f) \\ 
    \pi(g) \hat{f} &= \operatorname{\textsc{dec}}(z_{\text{inv}}, g\cdot z_{\text{eq}})
\end{align*}
\end{restatable}

\section{Related Work} \label{sec:related_work}

\paragraph{Group equivariant neural networks} 
The equivariant subsampling/upsampling that we propose deals with feature maps (functions) defined on the space of the group $G$ or its subgroups $K$, which transform under the \textit{regular representation} with the group action. Hence our equivariant subsampling/upsampling is compatible with \textit{lifting-based} group equivariant neural networks defined on discrete groups \citep{cohen2016group, hoogeboom2018hexaconv, romero2019co, romero2020attentive} that define a mapping between feature maps on $G$. We also discuss the extension of group equivariant subsampling to be compatible with those defined on continuous/Lie groups \citep{cohen2018spherical, esteves2018learning, finzi2020generalizing, bekkers2020b, hutchinsonlelanzaidi2020lietransformer} in Section \ref{sec:future_work}.
This is in contrast to group equivariant neural networks that do not use lifting and use \textit{irreducible representations}, defining mappings between feature maps on the input space $\mathbf{X}$. \citep{cohen2016steerable, worrall2017harmonic, thomas2018tensor, kondor2018clebsch, weiler2018learning, weiler20183d, weiler2019general, esteves2020spin, fuchs2020se}.

\paragraph{Coset pooling} In particular, \cite{cohen2016group} propose \textit{coset pooling}, which is also a method for equivariant subsampling. Here a feature map $f$ on $G$ is mapped onto a feature map $\Phi(f)$ on $G/K$ (as opposed to $K$, for our equivariant subsampling) as follows:
\begin{equation}
\Phi (f)(gK) = \textsc{pool}_{k \in K} f(gk) 
\end{equation}
such that the feature values on the coset $gK$ are pooled. For the 1D-translation case, where $G=\mathbb{Z}, K=c\mathbb{Z}$, this amounts to pooling over every $c$th pixel, which disrupts the locality of features as opposed to our equivariant subsampling that preserves locality, and hence is more suitable to use with convolutions for translation equivariance. See \Cref{fig:coset_pooling} for a visual comparison. As such, the $p4$-CNNs in \cite{cohen2016group} use standard max pooling with stride=2 rather than coset pooling for $\mathbb{Z}^2$, and coset-pooling is only used in the final layer to pool over feature maps across 90-degree rotations, achieving exact rotation equivariance but imperfect translation equivariance. In our work, we use translation equivariant subsampling in the earlier layers and rotation equivariant subsampling in the final layers to achieve exact roto-translation equivariance. 

\begin{figure}[!ht]
  \centering
    \includegraphics[width=0.65\linewidth]{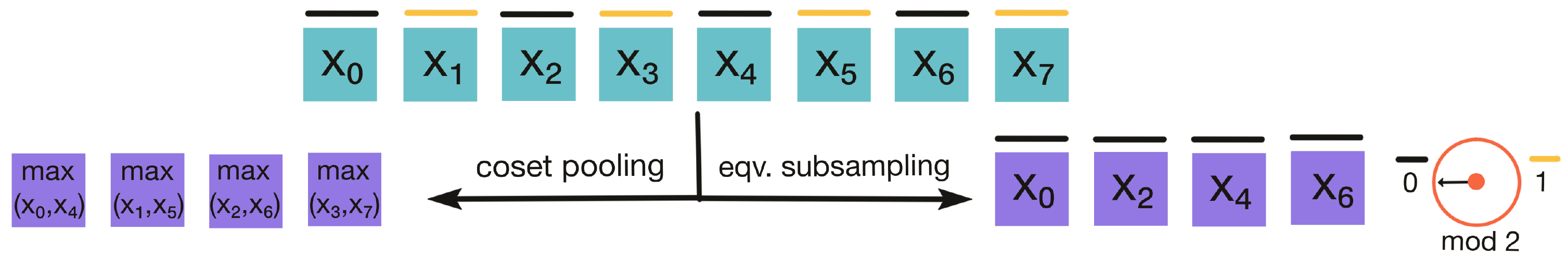}
    \caption{Coset (max) pooling vs. equivariant subsampling.}
    \label{fig:coset_pooling}
\end{figure}

\vspace*{-0.3cm}
\paragraph{Unsupervised disentangling and object discovery} GAEs produce equivariant ($z_{\text{eq}}$) and invariant ($z_{\text{inv}}$) representations, effectively separating position and pose information with other semantic information. This relates to unsupervised disentangling \citep{higgins2016beta,chen2018isolating,kim2018disentangling,zhao2017infovae} where different factors of variation in the data are separated in different dimensions of a low-dimensional representation. However unlike equivariant subsampling, there is no guarantee of any equivariance in the low-dimensional representation, making the resulting disentangled representations less interpretable.
Works on unsupervised object discovery \citep{burgess2019monet,greff2019multi,engelcke2020genesis,locatello2020object} learn object-centric representations, and we showcase GAEs in MONet \citep{burgess2019monet} where we replace their VAE with a V-GAE in order to separate position and pose information and learn more interpretable representations of objects in a data-efficient manner.

\paragraph{Shift-invariance in CNNs} As early as \cite{119725}, it has been discussed that shift-invariance cannot hold for conventional subsampling. Although standard subsampling operations such as pooling or strided convolutions are not \textit{exactly} shift invariant, they do not prevent strong performance on classification tasks \citep{10.1007/978-3-642-15825-4_10}. Nonetheless, \cite{zhang2019making} integrates anti-aliasing to improve shift-invariance, showing that it leads to better performance and generalisation on classification tasks. \cite{chaman2020truly} explore a similar strategy to our equivariant subsampling by partitioning feature maps into polyphase components and select the component with the highest norm. However, unlike the proposed group equivariant subsampling/upsampling which tackle general equivariance for arbitrary discrete groups, both works focus only on translation invariance. 

\section{Experiments} \label{sec:experiments}

\vspace*{-0.1cm}
In this section, we compare the performance of GAEs with equivariant subsampling to their non-equivariant counterparts that use standard subsampling/upsampling in object-centric representation learning. We show that GAEs give rise to more interpretable representations that show better sample complexity and generalisation than their non-equivariant counterparts.

\vspace*{-0.1cm}
\paragraph{Models and Baselines} (G-)Convolutional autoencoders (G)ConvAE are composed of alternating (G-)convolutional layers and subsampling/upsampling operations with a final MLPs applied to the flattened feature maps. We categorize models by the types of equivariance preserved by the convolutional layers. We consider three different discrete symmetry groups: $p1$ (only translations), $p4$ (composition of translations and $90$ degree rotations), $p4m$ (composition of translations, $90$ degree rotations and mirror reflection). The baseline models are: ConvAE-$p1$ (standard convolutional autoencoders), GConvAE-$p4$, GConvAE-$p4m$, where the corresponding equivariance is preserved in the (G-)convolutional layers but not in the subsampling/upsampling operations. The equivariant counterparts of these baseline models are GAE-$p1$, GAE-$p4$, GAE-$p4m$, where the subsampling/upsampling operations are also equivariant. For baseline models, we use a scale factor of $2$ for all subsampling/upsampling layers.
For GAEs, we subsample first the translations, then rotations, followed by reflections, all with scale factor 2. e.g. for GAE-$p4m$, the feature maps at each layer are defined on the following chain of nested subgroups:
$\mathbb{Z}^2 \rtimes (\mathsf{C}_4 \rtimes \mathsf{C}_2) \geq (2\mathbb{Z})^2 \rtimes (\mathsf{C}_4 \rtimes \mathsf{C}_2) \geq (4\mathbb{Z})^2 \rtimes (\mathsf{C}_4 \rtimes \mathsf{C}_2) \geq (8\mathbb{Z})^2 \rtimes (\mathsf{C}_4 \rtimes \mathsf{C}_2) \geq (16\mathbb{Z})^2 \rtimes (\mathsf{C}_2 \rtimes \mathsf{C}_2) \geq \{e\}$. 
As in \cite{cohen2016group}, we rescale the number of channels such that the total number of parameters of these models roughly match each other.

\paragraph{Data} To demonstrate basic properties of GAEs and compare sample complexity under the single object scenario, we use Colored-dSprite \citep{dsprites17} and a modification of FashionMNIST \citep{xiao2017/online}, where we first apply zero-padding to reach a size of $64\times 64$, followed by random shifts, rotations and coloring. For multi-object datasets, we use Multi-dSprites \citep{multiobjectdatasets19} and CLEVR6 which is a variant of CLEVR \citep{johnson2017clevr} with up to $6$ objects. All input images are resized to a resolution of $64\times 64$. 

See \Cref{sec:app_implementation} and our reference implementation \footnote{\url{https://github.com/jinxu06/gsubsampling}} for more details on hyperparameters and data preprocessing. Our implementation is built upon open source projects \cite{harris2020array,NEURIPS2019_9015,Yadan2019Hydra,e2cnn,engelcke2020genesis,Hunter:2007,Waskom2021}. 

\subsection{Basic Properties: Equivariance, Disentanglement and Out-of-Distribution Generalization}

\paragraph{Equivariance} The encoder-decoder pipeline in GAEs is exactly equivariant. In \Cref{fig:equivariance_disentanglement}, we train GAE-$p4m$ on $6400$ examples from Colored-dSprites, and we show how to manipulate reconstructions by manipulating the equivariant representation $z_{\text{eq}}$ (left). If an image $x$ is encoded into $[z_{\text{inv}}, z_{\text{eq}}]$, then decoding $[z_{\text{inv}}, g\cdot z_{\text{eq}}]$ will give $g\cdot \hat{x}$ where $\hat{x}$ is the reconstruction of $x$.

\begin{figure}[t]
     \centering
     \begin{subfigure}[b]{0.4\textwidth}
         \centering
         \includegraphics[width=\textwidth]{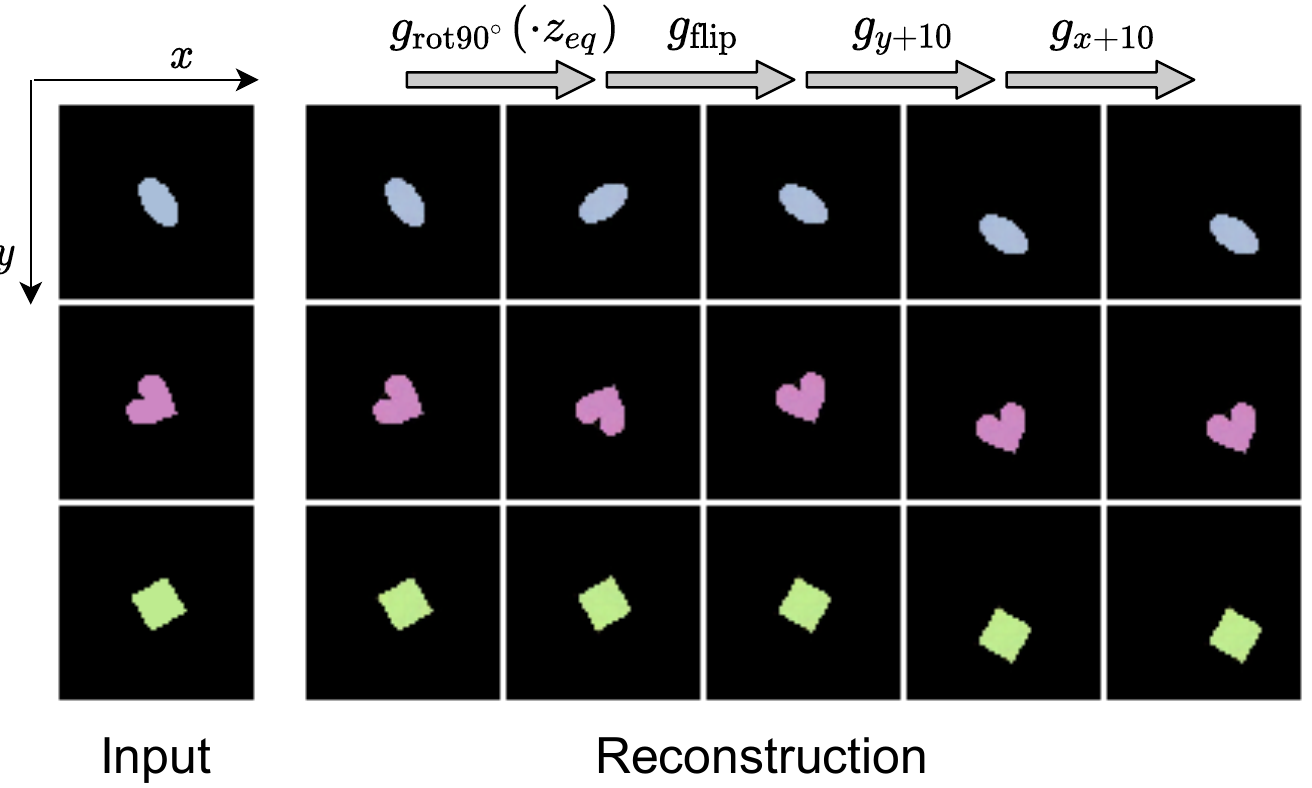}
     \end{subfigure}
     \hfill
     \begin{subfigure}[b]{0.5\textwidth}
         \centering
         \includegraphics[width=\textwidth]{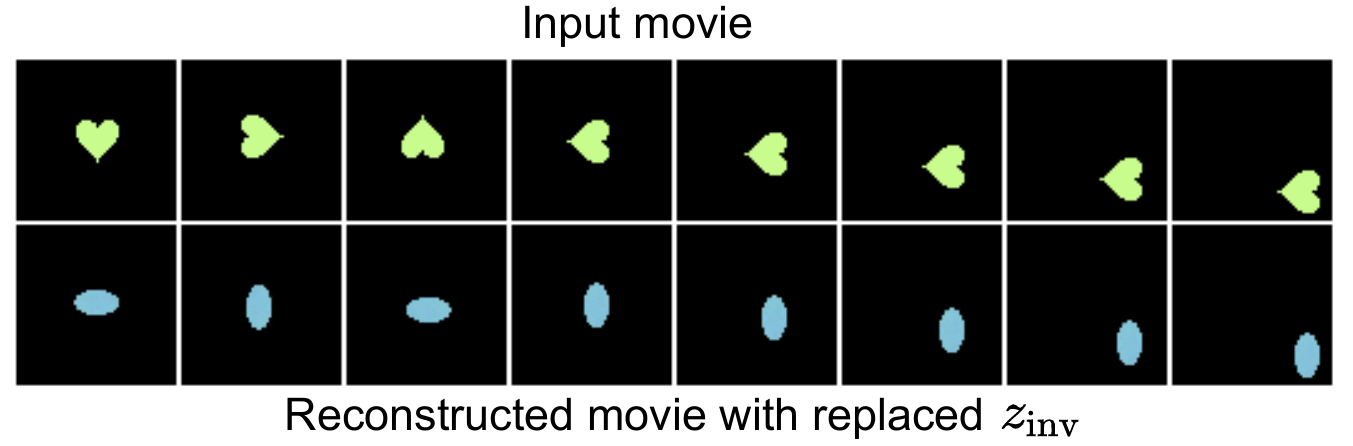}
     \end{subfigure}
        \caption{(Left) Manipulating reconstructions by modifying the equivariant part $z_{\text{eq}}$. The second column are the original reconstructions, which match the inputs well. The subsequent columns are reconstructions decoded from modified $z_{\text{eq}}$. We transform $z_{\text{eq}}$ with a sequence of group elements, and show the resulting reconstructions. (Right) Manipulating reconstruction shape by modifying $z_{\text{inv}}$.
        \vspace{-6pt}}
        \label{fig:equivariance_disentanglement}
\end{figure}

\paragraph{Disentanglement} The learned representations in GAEs are disentangled into an invariant part $z_{\text{inv}}$ and an equivariant part $z_{\text{eq}}$. In \Cref{fig:equivariance_disentanglement} (left), we vary the equivariant part while the invariant part remains the same. In \Cref{fig:equivariance_disentanglement} (right), we show the frames of a movie of a heart, and show its reconstruction after replacing $z_{\text{inv}}$ representing a heart with that of an ellipse. Note that the ellipse shape undergoes the same sequence of transformations as the heart.

\paragraph{Out-of-distribution generalisation} GAEs can generalise to data with unseen object locations and poses. We train an GAE-$p4$ on $6400$ constrained training examples, where we only use examples with locations in the top-left quarter and orientations within $[0, 90]$ degrees, as shown in~\Cref{fig:ood}. During test time, we evaluate mean squared error (MSE) of reconstructions on unfiltered test data to see how models generalise to unseen location and poses. 
Both ConvAE-$p1$ and GConvAE-$p4$ cannot generalise well to object poses out of their training distribution. In contrast, GAE-$p1$ generalise to any locations without performance degradation but not to unseen orientations, while GAE-$p4$, which encodes both translation and rotation equivariance, generalises well to all locations and orientations. We only use heart shapes for evaluation, because the square and ellipse have inherent symmetries.

\begin{figure}[t]
  \centering
    \includegraphics[width=0.8\linewidth]{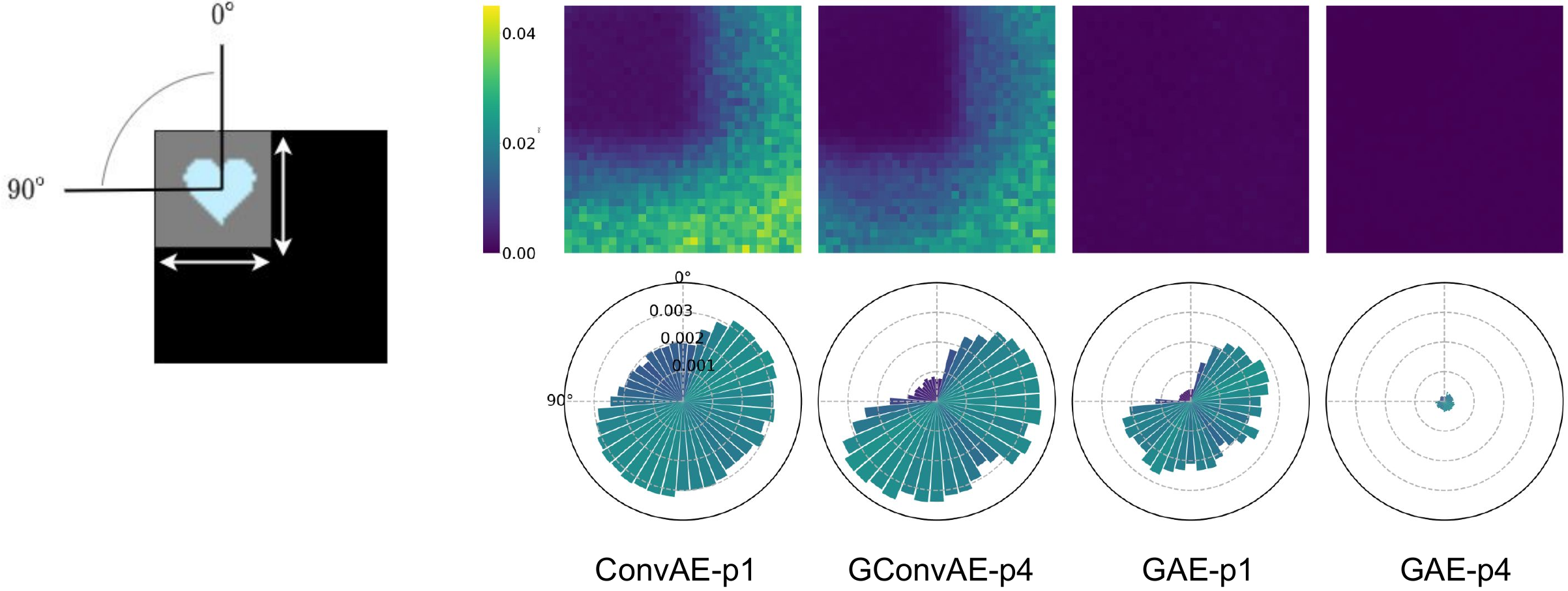}
    \caption{Generalisation to out-of-distribution object locations and poses. During training, we constrain shapes to be in the top-left quarter, and the orientation to be always less than $90$ degrees. On the right, we compare the error of reconstructions of different models generalise on objects at unseen locations in the first row, and how they generalise to unseen orientations in the second row.
    \vspace{-12pt}}
    \label{fig:ood}
\end{figure}

\begin{wrapfigure}[16]{r}{0.36\textwidth}
\vspace{-40pt}
    \includegraphics[width=0.36\textwidth]{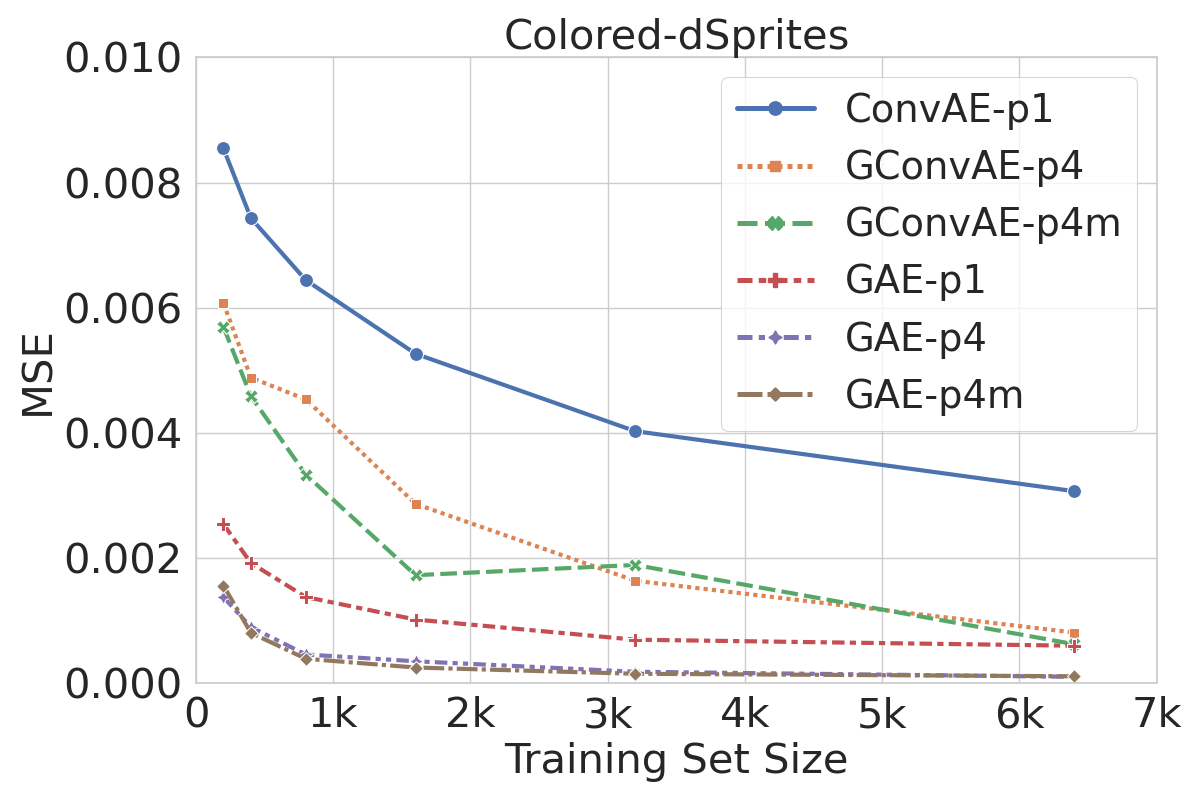}
    \includegraphics[width=0.36\textwidth]{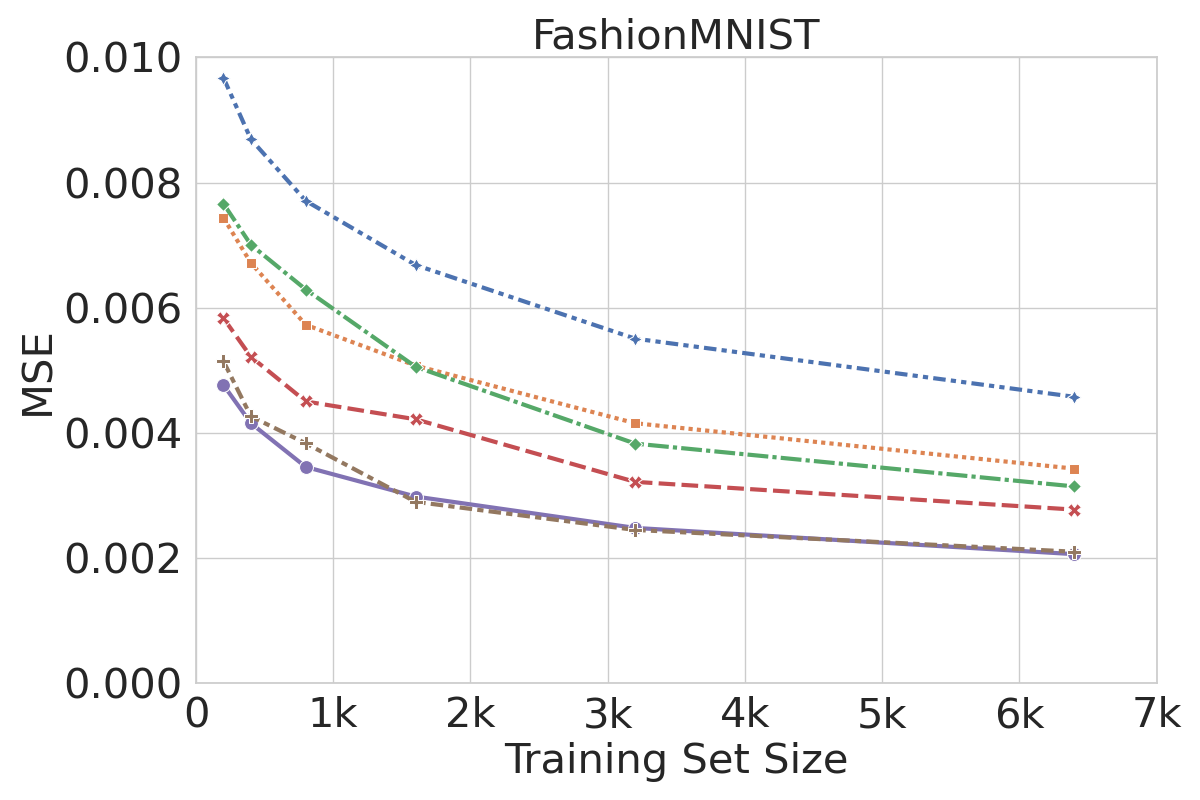}
    \caption{Reconstruction error on single object datasets} \label{fig:sample_complexity_single}
\end{wrapfigure}
\subsection{Single Object}
Since GAEs are fully equivariant and can generalize to unseen object poses, it is natural to conjecture that such models can significantly improve data efficiency when symmetry-transformed data points are also plausible samples from the data distribution. We test this hypothesis on Colored-dSprites and transformed FashionMNIST, and the results are shown in \Cref{fig:sample_complexity_single}. On both datasets, equivariant autoencoders significantly outperform their non-equivariant counterparts for all considered training set sizes. In fact, as shown in the figure, equivariant models trained with a smaller training set size is often comparable to baseline models trained on a larger training set. Furthermore, the results demonstrate that it is beneficial to consider symmetries beyond translations in these problems: for both non-equivariant and equivariant models, variants that encode rotation and reflection symmetries consistently show better performance compared to models that only consider the translation symmetry.

\subsection{Multiple Objects}

\looseness=-1
In multi-object scenes, it is often more interesting to consider local symmetries associated with objects rather than the global symmetry for the whole image. To exploit object symmetries in image data, one needs to first discover objects and separate them from the background, which is a challenging problem on its own. Currently, GAEs do not have inherent capability to solve these problems. In order to investigate whether our models could improve data efficiency in multi-object settings, we rely on recent work on unsupervised object discovery and only use GAEs to model object components. More specifically, we explored replacing component VAEs in MONet \citep{burgess2019monet} with V-GAEs (probabilistic version of our GAEs, where a standard Gaussian prior is put on $z_{\text{inv}}$ and $z_{\text{eq}}$ remains deterministic), and train models end-to-end. Again we study the low data regime to show results on data efficiency.

\begin{table*}[t]  
  \centering
  \footnotesize
  \begin{threeparttable}[]
  \caption{Reconstruction error MSE ($\times 10^{-3}$) (mean(stddev) across $5$ seeds) on multi-object datasets}
  \tabcolsep=0.09cm 
  \begin{tabular}{l|ccc|ccc}
    \toprule 
    Dataset & \multicolumn{3}{c|}{Multi-dSprites} & \multicolumn{3}{c}{CLEVR6}  \\
    \toprule
    Training Set Size & 3200 & 6400 & 12800 & 3200 & 6400 & 12800 \\
    \toprule
    MONet & $2.661(0.382)$ & $1.385(0.235)$ & $0.326(0.076)$ & $0.673(0.059)$ & $0.562(0.057)$ & $0.546(0.056)$\tnote{1} \\
    MONet-GAE-$p1$ & $0.659(0.103)$ & $0.359(0.025)$ & $0.264(0.042)$ & $0.473(0.064)$ & $0.432(0.052)$ & $0.388(0.016)$ \\
    MONet-GAE-$p4$ & $0.563(0.195)$ & $0.317(0.060)$ & $0.231(0.067)$ & $0.461(0.025)$ & $0.414(0.022)$ & $0.413(0.018)$ \\
    \bottomrule
  \end{tabular}
  \vspace*{-0.1cm}
  \label{tb:multi_object_mse}
  \end{threeparttable}
\end{table*}
\begin{table*}[t]  
  \centering
  \footnotesize
  \begin{threeparttable}[]
  \caption{Foreground segmentation performance in terms of ARI (mean(stddev) across $5$ seeds)}
  \tabcolsep=0.09cm 
  \begin{tabular}{l|ccc|ccc}
    \toprule 
    Dataset & \multicolumn{3}{c|}{Multi-dSprites} & \multicolumn{3}{c}{CLEVR6}  \\
    \toprule
    Training Set Size & 3200 & 6400 & 12800 & 3200 & 6400 & 12800 \\
    \toprule
    MONet & $0.597(0.022)$ & $0.747(0.049)$ & $0.891(0.009)$ & $0.829(0.055)$ & $0.878(0.023)$ & $0.865(0.033)$\tnote{1} \\
    MONet-GAE-$p1$ & $0.762(0.049)$ & $0.823(0.042)$ & $0.889(0.013)$ & $0.921(0.015)$ & $0.917(0.032)$ & $0.920(0.025)$ \\
    MONet-GAE-$p4$ & $0.753(0.089)$ & $0.833(0.072)$ & $0.902(0.025)$ & $0.878(0.055)$ & $0.914(0.012)$ & $0.910(0.011)$ \\
    \bottomrule
  \end{tabular}
  \label{tb:multi_object_ari}
  \begin{tablenotes}
        \item[1] \textit{We excluded $2$ outliers here as the baseline MONet occasionally fails during late-phase training.}
  \end{tablenotes}
  \vspace*{-0.3cm}
  \end{threeparttable}
\end{table*}

\looseness=-1
We train models on Multi-dSprites and CLEVR6 with training set sizes $3200$, $6400$ and $12800$. We consider two evaluation metrics: mean squared error (MSE) to measure the overall reconstruction quality, and adjusted rand index (ARI), which is a clustering similarity measure ranging from $0$ (random) to $1$ (perfect) to measure object segmentation. As in \cite{burgess2019monet}, we only use foreground pixels to compute ARI. Component VAEs in MONet use spatial broadcast decoders \citep{watters2019spatial} that broadcast the latent representation to a full scale feature map before feeding them into the decoders, and the decoders therefore do not need upsampling. It has the implicit effect of encouraging the smoothness of the decoder outputs. To encourage similar behaviour, we add average pooling layers with stride $1$ and kernel size $3$ to our equivariant decoders. As shown in \Cref{tb:multi_object_mse}, using GAEs to model object components significantly improves reconstruction quality, which is consistent with our findings in single-object scenario. As shown in \Cref{tb:multi_object_ari}, using GAEs to model object components also leads to better object discovery in the low data regimes, but this advantage seems to diminish as the dataset becomes sufficiently large.

\section{Conclusions, Limitations and Future Work}
\label{sec:future_work}

\paragraph{Conclusions} We have proposed subsampling/upsampling operations that \textit{exactly} preserve translation equivariance, and generalised them to define \textit{exact} group equivariant subsampling/upsampling for discrete groups. We have used these layers in GAEs that allow learning low-dimensional representations that can be used to reliably manipulate pose and position of objects, and further showed how GAEs can be used to improve data efficiency in multi-object representation learning models.

\paragraph{Limitations and Future work} 
Although the equivariance properties of subsampling layers also hold for Lie groups, we have not discussed the practical complexities that arise with the continuous case, where feature maps are only defined on a finite subset of the group rather than the whole group. We leave this as important future work, as well as application of equivariant subsampling for tasks other than representation learning where equivariance/invariance is desirable e.g. object classification, localization. Another limitation is that our work focuses on global equivariance, like most other works in the literature. An important direction is to extend to the case of local equivariances e.g. object-specific symmetries for multi-object scenes.

\begin{ack}

We would like to thank Adam R. Kosiorek for valuable discussion. We also thank Lewis Smith, Desi Ivanova, Sheheryar Zaidi, Neil Band, Fabian Fuchs, Ning Miao, and Matthew Willetts for providing feedback on earlier versions of the paper. JX gratefully acknowledges funding from Tencent AI Labs through the Oxford-Tencent Collaboration on Large Scale Machine Learning. 
\end{ack}

\medskip

{
\small
\bibliographystyle{apalike}  
\bibliography{gsubsampling}
}

\newpage 

\appendix

\section{Preliminaries} \label{sec:app_preliminaries}

\subsection{Group, Coset and Quotient Space}

A \emph{group} $G$ is a set of elements equipped with a binary operation (denoted as $\cdot$) that satisfies the following group axioms: 
\begin{enumerate}
    \item (Closure) For all $a, b\in G$, $a\cdot b \in G$.
    \item (Associative) For all $a,b,c \in G$, $(a \cdot b) \cdot c = a \cdot (b \cdot c)$.
    \item (Identity element) There exists an identity element $e$ in $G$ such that, for any $a \in G$ we have $e\cdot a = a \cdot e = a$.
    \item (Inverse element) For each $a\in G$, there exists an element $b\in G$ such that $a \cdot b = b \cdot a = e$ where $e$ is the identity element.
\end{enumerate}
The centered dot $\cdot$ can sometimes be omitted if there is no ambiguity.

In this work, we are mainly interested in symmetry groups where each group element is associated with a symmetry of a pattern, which is a transformation that leaves the pattern invariant. In symmetry groups, the binary operation corresponds to composition of transformations. 

A subset $H$ contained within $G$ is a \emph{subgroup} of $G$ if it forms a group on its own under the same binary operation.
Given a subgroup $H$ and an arbitrary group element $g\in G$, one can define \emph{left cosets} of $H$ as follows:
\begin{align*}
    gH &= \{g\cdot h \;|\; h\in H\} 
\end{align*}
The left cosets of $H$ form a partition of $G$ for any choice of $H$, i.e. the union of all cosets is $G$ and all cosets defined above are either identical or have empty interception. The set of all left cosets is called the \emph{quotient space} and is denoted as $G/H = \{gH \;|\; g \in G \}$.

As an example, all integers $\mathbb{Z}$ under addition forms a group and all multiples of $n$, denoted as $n \mathbb{Z}$ is a subgroup of $\mathbb{Z}$. For any integer $k \in \mathbb{Z}$, the set $n \mathbb{Z} + k$ containing all integers that has the remainder as $k$ divided by $n$, is a coset of $n \mathbb{Z}$. There are $n$ distinct cosets like this, and they form the quotient space $\mathbb{Z}/n \mathbb{Z}$.

\subsection{Group Homomorphism, Group Actions and Group Equivariance}

Given two groups $(G, \cdot_{G})$ and $(H, \cdot_{H})$, a \emph{group homomorphism} from $G$ to $H$ is function $f:G\rightarrow H$ such that for any $u, v \in G$
\begin{equation*}
    f(u \cdot_G v) = f(u) \cdot_H f(v).
\end{equation*}
It is a special mapping between two groups that is compatible with group structures. If $f$ is an one-to-one mapping, we call it a \emph{group isomorphism}. Two groups $G_1$ and $G_2$ are isomporphic if there is an isomorphism between them, and this is written as $G_1 \cong G_2$.

A \emph{group action} is a group homomorphism from a given group $G$ to the group of transformations on a space $\mathbf{X}$. We say the group $G$ acts on the space $\mathbf{X}$ and the transformation corresponding to $g\in G$ is a bijection on $\mathbf{X}$ that maps $x$ to $g\cdot x$.

If the group actions of $G$ on spaces $\mathbf{X}$ and $\mathbf{Y}$ are both defined, a function $f:\mathbf{X}\rightarrow \mathbf{Y}$ is said to be \emph{group equivariant} if 
\begin{align*}
    g \cdot f(x) = f(g\cdot x)
\end{align*}

\subsection{Homogeneous Spaces and Lifting Feature Maps} \label{sec:app_lifting}

If the action of a group $G$ on the space $\mathbf{X}$ is defined, and the action is transitive (i.e. $\forall x,x' \in \mathbf{X},\exists g\in G$, s.t.  $x' = g \cdot x$), we refer to $\mathbf{X}$ as being a homogeneous space for $G$. There is a natural one-to-one correspondence between the homogeneous space $\mathbf{X}$ and disjoint subsets of the group $G$. Given an arbitrary origin $x_0\in \mathbf{X}$, $H=\{g\in G| g\cdot x_0 = x_0\}$ is a subgroup of $G$, where $H$ is called the stabiliser of the origin. Because the group action on $\mathbf{X}$ is transitive, every element $x\in \mathbf{X}$ corresponds to a left coset in $s(x;x_0)H\in G/H$, where $s(x;x_0)$ is (any) group element that transforms $x_0$ to $x$. It can be shown that for $x,x'\in \mathbf{X}, x\neq x'$, $s(x;x_0)$ and $s(x';x_0)$ are disjoint.

Because spatial data is often represented as functions on the homogeneous space $f_{\mathbf{X}}: x_i \mapsto f_i$, while lifting-based group equivariant neural networks operate on feature maps defined on the group, there is usually an operation called \textit{lifting}, that maps the data to the feature space of functions on the group, before applying equivariant modules. Using the correspondence between $\mathbf{X}$ and the quotient space $G/H$, we can map each pair $(x_i,f_i)$ to the set $\{(g,f_i) | g\in s(x_i;x_0)H \}$. It can be seen as lifting the input feature map $f_{\mathbf{X}}: x_i \mapsto f_i$ to the feature map $\operatorname{\textsc{lift}}(f_{\mathbf{X}}): g \mapsto f_i$ for $g \in s(x_i;x_0)H$. In this work, we assume all input feature maps have been lifted to feature maps on the group.

\subsection{Wallpaper Groups}

Wallpaper groups categorise symmetries of repetitive patterns on a 2D plane. For simplicity, we only considered $3$ different types of wallpaper symmetry groups $p1$, $p4$, and $p4m$ in this work following \cite{cohen2016group}. These groups are named using the crystallographic notation, where $p$ standards for primitive cells, the next digit indicates the highest order of rotational symmetries, and $m$ stands for mirror reflection. All symmetries contained in these groups can be deduced from their name:
\begin{itemize}
    \item $p1$: All 2D integer translations.
    \item $p4$: All compositions of 2D integer translations and rotations by a multiple of $90$ degrees.
    \item $p4m$: All compositions of elements in $p4$ and the mirror reflection.
\end{itemize}

All three groups $p1$, $p4$, and $p4m$ can be constructed from basic additive groups of integers $\mathbb{Z}$ and cyclic groups $\mathsf{C}_n$ using the inner semi-direct product. Given a group $G$ with a normal subgroup $N$ (i.e. $\forall n\in N, g\in G$, $gng^{-1}\in N$), a subgroup $H$ (not necessarily a normal subgroup), if $G$ is the product of subgroups $G=NH=\{nh|n\in N, h\in H\}$, and $N \cap H=\{e\}$, we say $G$ is a inner semi-direct product of $N$ and $H$, written as $G = N \rtimes H$. Using semi-direct product, $p1$, $p4$, and $p4m$ can be expressed as:
\begin{align} \label{eq:decomposition}
    p1 &\cong \mathbb{Z}^2 \nonumber \\ 
    p4 &\cong \mathbb{Z}^2 \rtimes \mathsf{C}_4 \nonumber \\
    p4m &\cong \mathbb{Z}^2 \rtimes (\mathsf{C}_4 \rtimes \mathsf{C}_2)
\end{align}

If $G\cong N \rtimes H$, the binary and inverse operations for $G$ can be determined from its subgroups $N$ and $H$. We represent group elements in $G$ as a tuple $(n, h)$ where $n\in N$ and $h\in H$. Let $\phi_h(n)=hnh^{-1}$, the binary operation on $G$ can be given by:
\begin{align*}
    (n_1,h_1) \cdot (n_2,h_2) = (n_1 \phi_{h_1}(n_2), h_1 h_2)
\end{align*}
and the inverse for element in $G$ can also be derived from the above:
\begin{align*}
    (n,h)^{-1} = (\phi_{h^{-1}}(n^{-1}), h^{-1})
\end{align*}
These properties can be used to simplify the implementation of the considered groups $p1$, $p4$, and $p4m$ following the decomposition in \Cref{eq:decomposition}, and the operations for basic groups $\mathbb{Z}$ and $\mathsf{C}_n$ are easy to implement.

\subsection{Feature Maps in G-CNNs} \label{subsec:app_feature maps}

A general mathematical framework is introduced in \cite{cohen2019general} to specify convolutional feature spaces used in G-CNNs, and feature maps are treated as fields over a homogeneous space. It covers most previous works on equivariant neural networks including \cite{cohen2016group,cohen2016steerable,s.2018spherical,weiler2019general}. Under this framework, one way to represent fields is through constrained functions defined on the whole symmetry group, also known as Mackey functions \citep{cohen2019general}.

Formally, let $G$ be a symmetry group, and $H\leq G$ together with $G$ determines the homogeneous space $G/H$. For a group representation $(\rho, V)$ of $H$, the action of the whole group $G$ on fields can be described by an induced representation $\pi =\text{Ind}_{H}^{G} \rho$, whose realisation depends on how we represent these fields. Below we specify the feature space $\mathcal{I}_M$ \footnote{$\mathcal{I}_M$ corresponds to $\mathcal{I}_G$ in \cite{cohen2019general}} for the Mackey function field representation discussed in \citep{cohen2018intertwiners,cohen2019general}:
\begin{align} \label{eq:mackey_functions}
    \mathcal{I}_M &=\{f:G\to V| f(gh) = \rho(h^{-1}) f(g),\forall g\in G, h\in H \} 
\end{align}
which forms a vector space.
Moreover, when $\rho$ is a \emph{regular representation}, which is the implicit choice of \cite{gens2014deep,kanazawa2014locally,dieleman2015rotation,dieleman2016exploiting,cohen2016group,marcos2016learning}, fields can also be represented as unconstrained functions on $G$ and the feature space can be written as 
\begin{align*}
    \mathcal{I}_G=\{f:G\to V'\}
\end{align*}
with $V'$ being a different vector space from $V$. If $\rho$ is a regular representation.

Feature maps are represented as functions on $G$ in both $\mathcal{I}_M$ and $\mathcal{I}_G$, even though $\mathcal{I}_M$ have additional conditions given in \Cref{eq:mackey_functions}. Moreover, the induced representation $\pi=\text{Ind}_{H}^{G} \rho$ for them have the same form:
\begin{align*}
    [\pi(u)f](g) = f(u^{-1} g)
\end{align*}

\section{Equivariant Subsampling and Upsampling} \label{sec:app_subsampling}

\subsection{Constructing $\Phi$} \label{sec:app_phi}

In \Cref{subsec:phi}, we provide a simple construction of the equivariant map $\Phi:\mathcal{I}_G\rightarrow G/K$ which gives the sampling indexes. The construction is a valid one if the argmax is unique. In practice one can insert arbitrary equivariant layers to $f$ before and after we take the norm $\|\cdot\|_1$ to avoid a non-unique argmax (see \Cref{sec:app_implementation}). However, in theory, there could be cases that the argmax is always non-unique. We discuss this case below and provide a more complex construction for it.

One cannot avoid a non-unique argmax in \Cref{eq:phi} when the input feature map $f\in \mathcal{I}_G$ has inherent symmetries, i.e. there exists $u\in G,u\neq e$, such that $f = \pi(u)f$. Assuming there is a unique argmax $g^{\ast}$ such that $g^{\ast} = \arg\max_{g\in G} \|f(g)\|_1$, we would have:
\begin{align*}
    f(u \cdot g^{\ast})=f(g^{\ast})=\max_{g\in G} \|f(g)\|_1
\end{align*}
Therefore $u \cdot g^{\ast}$ is also a valid argmax, hence the argmax is not unique. For example, when $f$ is a feature map representing a center-aligned circle, we would have $f = \pi(u) f$, where $u\in \mathit{O}(2)$ is associated with an arbitrary rotation around the center. One cannot find a unique argmax $g^{\ast}$ for this example, because the feature map would take the same function values at $u\cdot  g^{\ast}$.

Under the circumstance described above, the argmax operation would return a set of elements where each one attains the function's largest values. We denote it as $S^{\ast}=\arg\max_{g\in G} \|f(g)\|_1$, where $S^{\ast}$ is a subset of $G$. To obtain the sampling index (a coset) $pK$, we sample uniformly from the set $S^{\ast}$, and let $\Phi$ outputs $p K$ where $p \sim S^{\ast}$. In this case, the map $\Phi$ is still equivariant in distribution even though it is now a stochastic map.

Note that it is possible to consider more sophisticated solutions or even use learnable modules for $\Phi$, which we leave for future work.

\subsection{Multiple Subsampling Layers} \label{sec:app_multilayer}

\paragraph{Translation equivariant subsampling}

We can stack convolutional and translation equivariant subsampling layers to construct exactly translation equivariant CNNs. Unlike standard CNNs, each translation equivariant subsampling layer with a scale factor $c_k$ outputs a subsampling index $i_k$ in addition to the feature maps. Hence the equivariant representation output by the CNN with $L$ subsampling layers is a final feature map $f_L$ and a $L$-tuple of sampling indices $(i_1, . . . , i_L)$. 

In the multi-layer case, the $l$-th subsampling layer takes in a feature map $f$ on $\prod_{k=1}^{l-1} c_k \mathbb{Z}$ and outputs: 1) a feature map on $\prod_{k=1}^l c_k \mathbb{Z}$ and 2) a subsampling index $i_l \in \prod_{k=1}^{l-1} c_k \mathbb{Z} / \prod_{k=1}^l c_k \mathbb{Z} \cong \mathbb{Z} / c_l \mathbb{Z}$ given by:
\begin{align*}
    (\prod_{k=1}^{l-1} c_k) \cdot i_l = p_l = \Phi_c(f) = \bmod({\arg\max}_{x\in (\prod_{k=1}^{l-1} c_k) \mathbb{Z}} \|f(x)\|_1, \prod_{k=1}^l c_k) 
\end{align*}

This is equivalent to treating the input feature map $f$ as a feature map $f'$ defined on $\mathbb{Z}$ (i.e. mapping the support of $f$ from $\prod_{k=1}^{l-1} c_k \mathbb{Z}$ to $\mathbb{Z}$ via division by $\prod_{k=1}^{l-1} c_k$), and the subsampling layer outputting: 1) a feature map on $c_l \mathbb{Z}$ and 2) a subsampling index $i_l \in \mathbb{Z} / c_l \mathbb{Z}$ given by:
\begin{align*}
    i_l = \bmod({\arg\max}_{x\in \mathbb{Z}} \|f'(x)\|_1, c_l) 
\end{align*}

Hence the tuple $(i_1, . . . , i_L)$ that contains the sampling indices of all layers can be expressed equivalently as a single integer:
\begin{align*}
    r_{\text{eq}} = \sum_{l=1}^{L} p_l =  \sum_{l=1}^{L} (\prod_{k=1}^{l-1} c_k) \cdot i_l
\end{align*}
where $r_{\text{eq}} \in \mathbb{Z}/(\prod_{k=1}^L c_k) \mathbb{Z}$. Note that the conversion between $r_{\text{eq}}$ and $(i_1, . . . , i_L)$ can be seen as the conversion between \emph{mixed radix notation} and decimal notation. Mixed radix notation is a mixed base numeral system where the numerical base varies from position to position, as opposed to base-n systems that have the same base for all positions\footnote{A commonly used example of mixed radix notation is to express time, where e.g. 12:34:56 has a base
of 24 for the hour digit, base 60 for the minute digit and base 60 for the second digit.}.
Thus there is an one-to-one correspondence between the two. Moreover, when the input feature map is translated to the right by $t \in \mathbb{Z}$, $r_{\text{eq}}$ would become $\bmod(r_{\text{eq}}+t, \prod_{k=1}^L c_k)$. See the statement of this result for the general group case in \Cref{thm:conversion} and its proof in \Cref{sec:conversion_proof}.

\paragraph{Group equivariant subsampling}

Similarly, given an input feature map $f\in \mathcal{I}_G$, we can construct CNNs/G-CNNs with multiple equivariant subsampling layers by specifying a sequence of nested subgroups $G=G_0 \geq G_1 \geq \dots \geq G_L$. The $l$-th subsampling layer takes in a feature map on $G_{l-1}$, outputs a feature map on $G_l$ and a sampling index $p_l G_l \in G_{l-1}/G_{l}$. Formally, the $l$-th subsampling layer can be written as:
\begin{align*}
    S_b{\downarrow}_{G_l}^{G_{l-1}}: \mathcal{I}_{G_{l-1}} \rightarrow \mathcal{I}_{G_l} \times G_{l-1}/G_{l}
\end{align*}
The equivariant representation output by the CNNs/G-CNNs with $L$ subsampling layers is a feature map in $f_L \in G_L$ and a $L$-tuple $(p_1 G_1,p_2 G_2,\dots,p_L G_L)$.

Similar to the 1D translation case, the sampling index tuple $(p_1 G_1,p_2 G_2,\dots,p_L G_L)$ can be expressed equivalently as a single element in the quotient space $G/G_L$:
\begin{align} \label{eq:nu}
    r_{\text{eq}} = (\bar{p}_1 \bar{p}_2 \dots \bar{p}_L) G_L = \nu(p_1 G_1,p_2 G_2,\dots,p_L G_L)
\end{align}
where $\bar{p}_l$ denote the coset representive for the quotient space $G_{l-1}/G_{l}$.
$\nu$ is a bijection from $r_{\text{eq}}$ to the tuple, whose inverse can be computed by the following recursive procedure:
\begin{align} \label{eq:recursive}
    p'_1 G_L &= r_{\text{eq}} \nonumber \\
    p'_l &= \bar{p}_{l-1}^{\prime-1} \cdot  p'_{l-1} \nonumber \\
    (p'_1 G_1,p'_2 G_2,\dots,p'_L G_L) &= \nu^{-1}(r_{\text{eq}})
\end{align}

\begin{restatable}[]{proposition}{conversion}
\label{thm:conversion}
$\nu^{-1}$ is the inverse of $\nu$, hence $\nu$ is bijective. And $ \forall u\in G$ we have:
\begin{align*}
    u\cdot \nu(p_1 G_1,p_2 G_2,\dots,p_L G_L) = \nu(u \cdot (p_1 G_1,p_2 G_2,\dots,p_L G_L)).
\end{align*}
\end{restatable}

\section{Group Equivariant Autoencoders} \label{sec:app_autoencoders}

In \Cref{sec:app_multilayer} we discussed that we can stack multiple subsampling layers by specifying a sequence of nested groups $G=G_0\geq G_1 \geq \dots \geq G_L$, and the CNN/G-CNNs with $L$ subsampling layers would produce a feature map on $G_L$ and a tuple $z_{\text{eq}} = (p_1 G_1, p_2 G_2, \dots, p_L G_L)$. Furthermore, we know from \Cref{thm:conversion} that there is an one-to-one correspondence between the tuple representation $z_{\text{eq}}$ and the single group element representation $r_{\text{eq}} = \nu(z_{\text{eq}})\in G/G_L$. For group equivariant autoencoders, we specify a sequence of subgroups but with $G_L=\{e\}$. In this case, $r_{\text{eq}}$ would simply become a group element in $G$. And the group action simplifies to left-multiplying the corresponding group elements.

Although one can simply use the tuple output by the encoder to perform upsampling in the decoder (and hence use the same sequence of nested subgroups), this is not strictly necessary as one can use a different sequence of nested subgroups for the decoder and obtain the tuple using the decomposition procedure in \Cref{eq:recursive}. Moreover, for more efficient implementation of GAEs, one does not need to pass through $\Phi$ in \Cref{eq:phi} for every subsampling layer. It would suffice to obtain the tuple of subsampling indexes from the first subsampling layer using:
\begin{align} \label{eq:phi_1}
     (p_1 G_1, p_2 G_2, \dots, p_L G_L) = \nu^{-1}(\arg\max_{g\in G} \|f(g)\|_1)
\end{align}

\section{Proofs} \label{sec:app_proofs}

\subsection[]{Proof of Lemma~\ref{thm:gaction} (See page~\pageref{thm:gaction})}
\gaction*
\begin{proof}
Since $\bar{p}$ and $\overline{up}$ are coset representives for $p K$ and $(up) K$, we can let $p=\bar{p}k_p$, $up=\overline{up}k_{up}$, where $k_p,k_{up}\in K$. From \Cref{eq:group_action}, note that $\bar{p}'=\overline{up}=upk_{up}^{-1}$. Hence
\begin{align} \label{eq:p_bar_inv_up}
    \bar{p}'^{-1}u\bar{p} &= (upk_{up}^{-1})^{-1} u (pk_{p}^{-1}) \nonumber \\ 
    &= k_{up}p^{-1}u^{-1} u p k_{p}^{-1} \nonumber \\ 
    &= k_{up} k_{p}^{-1} \in K
\end{align}
Hence $\pi'(u)$ (as defined in \Cref{eq:group_action}) defines a transformation from the space $\mathcal{I}_K \times G/K$ to itself.

To prove $\pi'$ is a group action, we would like to show that for all $u, u'\in G$
\begin{align*} 
    \pi'(u') (\pi'(u) [f_b,\;p K]) = \pi'(u' u) [f_b,\;p K]
\end{align*}

Let $[f'_b,\;p' K] = \pi'(u) [f_b,\;p K]$ and $[f''_b,\;p'' K] = \pi'(u' u) [f_b,\;p K]$, by the definition of $\pi'$ in \Cref{eq:group_action}, we have
\begin{align*}
    p'' K &= ((u' u) p) K = u' (up K) = u' (p' K)
\end{align*}
and 
\begin{align*}
    f''_b &= \pi(\bar{p}''^{-1} (u' u) \bar{p}) f_b = \pi(\bar{p}''^{-1} u' \bar{p}') \pi(\bar{p}'^{-1} u \bar{p}) f_b = \pi(\bar{p}''^{-1} u' \bar{p}') f'_b.
\end{align*}
Hence 
\begin{align*} 
    [f''_b,\;p'' K] = \pi'(u') [f'_b,\;p' K] 
\end{align*}

It is easy to also check that 
\begin{align*}
    [f_b,\;p K] = \pi'(e) [f_b,\;p K]
\end{align*}
Therefore $\pi'$ defines a valid group action.
\end{proof}

\subsection[]{Proof of Proposition~\ref{thm:equivariance} (See page~\pageref{thm:equivariance})}
\equivariance*
\begin{proof}

We first define a \emph{restrict} operation on $f\in \mathcal{I}_G$ and an \emph{extend} operation on $f_1\in \mathcal{I}_K$:
\begin{align*}
    f {\downarrow}_K^G(k) &= f(k), \hspace{4mm} k \in K \\
    f_1{\uparrow}_K^G(g) &=  
    \begin{cases}
      f_1(g) & g \in K \\
      \bm{0} & g \notin K
    \end{cases} 
\end{align*}
where $f {\downarrow}_K^G\in \mathcal{I}_K$ and $f_1{\uparrow}_K^G \in \mathcal{I}_G$.

Recall that $s:G/K\rightarrow G$ is a function choosing a coset representive $\bar{p}$ for each coset $p K$. Using the restrict operation, the subsampling operation $S_b{\downarrow}_K^G(f;\Phi)$ in \Cref{eq:subsampling} can equivalently be described as:
\begin{align*}
    p K &= \Phi(f) \nonumber \\
    f_b &= [\pi(\bar{p}^{-1})f]{\downarrow}_K^G \nonumber \\
    [f_b,\;p K] &= S_b{\downarrow}_K^G(f; \Phi) 
\end{align*}
And the upsampling operation $S_u{\uparrow}_K^G$ can be rewritten using the extend operation as:
\begin{align*}
    f_u &= S_u{\uparrow}_K^G(f_1, p K) = \pi(\bar{p})(f_1{\uparrow}_K^G)
\end{align*}

For any $u\in G$ let $f' = \pi(u) f$ and $[f'_b,\;p' K] = \pi'(u) [f_b,\;p K]$ where $\pi$ and $\pi'$ are specified in \Cref{eq:action_pi} and \Cref{eq:group_action} respectively. Since $\Phi$ is equivariant, we have
\begin{align} \label{eq:eq3}
    \Phi(f') = \Phi(\pi(u) f) = u\cdot \Phi(f) = u\cdot p K = p' K
\end{align}

Recall that $\bar{p}'^{-1}u\bar{p}=k_{up} k_{p}^{-1}$ from \Cref{eq:p_bar_inv_up}. Hence $\bar{p}'^{-1}=k_{up} k_{p}^{-1} \bar{p}^{-1} u^{-1}$ and we have
\begin{align} \label{eq:eq4}
    [\pi(\bar{p}'^{-1})f']{\downarrow}_K^G &= [\pi(k_{up}k_p^{-1} \bar{p}^{-1} u^{-1})f']{\downarrow}_K^G \nonumber \\
    &= \pi(k_{up}k_p^{-1}) [\pi(\bar{p}^{-1}) \pi(u^{-1}) f']{\downarrow}_K^G \nonumber \\
    &= \pi(k_{up}k_p^{-1}) [\pi(\bar{p}^{-1}) f]{\downarrow}_K^G \nonumber \\
    &= \pi(k_{up}k_p^{-1}) f_b = f'_b
\end{align}

From \Cref{eq:eq3,eq:eq4}, $S_b{\downarrow}_K^G$ is equivariant, i,e.
\begin{align*}
    \pi'(u) S_b{\downarrow}_K^G(f; \Phi) = S_b{\downarrow}_K^G(\pi(u)f; \Phi) 
\end{align*}

For the upsampling operation, let $[f'_1, p'_1 K] = \pi'(u) [f_1, p_1 K]$ and $f'_u = \pi(u) f_u$. From \Cref{eq:group_action} we have $f'_1 = \pi(\bar{p}'^{-1}u\bar{p}) f_1$. Hence
\begin{align*}
    S_u{\uparrow}_K^G(f'_1, p' K) &= \pi(\bar{p}') f'_1{\uparrow}_H^G \\ 
    &= \pi(\bar{p}')[\pi(\bar{p}'^{-1}u\bar{p}) f_1]{\uparrow}_H^G \\
    &= \pi(\bar{p}') \pi(\bar{p}'^{-1}u\bar{p}) (f_1{\uparrow}_H^G) \\
    &= \pi(u\bar{p}) f_1 {\uparrow}_H^G = \pi(u) f_u = f'_u
\end{align*}
Therefore, $S_u{\uparrow}_K^G$ is equivariant, i.e.
\begin{align*}
    \pi(u) S_u{\uparrow}_K^G([f_1, p_1 K]) = S_u{\uparrow}_K^G(\pi'(u)[f_1, p_1 K])
\end{align*}
\end{proof}

\subsection[]{Proof of Proposition~\ref{thm:conversion} (See page~\pageref{thm:conversion})} \label{sec:conversion_proof}
\conversion*
\begin{proof}
Firstly, we prove that $\nu^{-1} \circ \nu$ is an identity map, i.e. $\nu^{-1} \circ \nu = \mathds{1}_z$.
Let $r_{\text{eq}} = \nu(p_1 G_1, p_2 G_2, \dots, p_L G_L) = (\bar{p}_1 \bar{p}_2 \dots \bar{p}_L) G_L$ and $(p'_1 G_1, p'_2 G_2, \dots, p'_L G_L) = \nu^{-1}(r_{\text{eq}})$. From \Cref{eq:recursive}, we know that $p'_1 G_L = r_{\text{eq}}$. Hence we can let $p'_1 = \bar{p}_1 \bar{p}_2 \dots \bar{p}_L g_L$ where $g_L\in G_L$. Since $(\bar{p}_2 \bar{p}_3 \dots \bar{p}_L) \in G_1$, for $l=1$ we have
\begin{align*}
    \bar{p}'_1 &= \bar{p}_1 \\ 
    p'_2 = \bar{p}_1^{\prime-1} \cdot  p'_1 &= \bar{p}_2 \bar{p}_3 \dots \bar{p}_L g_L 
\end{align*}
And recursively, for $l=1,\dots,L$ we would have
\begin{align*}
    \bar{p}'_l &= \bar{p}_l \\ 
    p'_{l+1} &= \bar{p}_{l+1} \dots \bar{p}_L g_L 
\end{align*}
Hence $(p_1 G_1,p_2 G_2,\dots,p_L G_L) = (p'_1 G_1,p'_2 G_2,\dots,p'_L G_L)$ and $\nu^{-1} \circ \nu = \mathds{1}_z$.

Secondly, we prove that $\nu \circ \nu^{-1}$ is also an identity map, i.e. $\nu \circ \nu^{-1} = \mathds{1}_r$.
Let $(p'_1 G_1, p'_2 G_2, \dots, p'_L G_L)=\nu^{-1}(r_{\text{eq}})$ and $r'_{\text{eq}}=\nu(p'_1 G_1, p'_2 G_2, \dots, p'_L G_L)$. From \Cref{eq:recursive}, we have $r_{\text{eq}} =p'_1 G_L$ and $p'_l=\bar{p}_{l-1}^{\prime-1} \cdot p_{l-1}$. Hence
\begin{align*}
    r_{\text{eq}} &= p'_1 G_L = \bar{p}'_1 p'_2 G_L = \cdots = \bar{p}'_1 \bar{p}'_2 \dots \bar{p}'_{L-1} p'_L G_L = \bar{p}'_1 \bar{p}'_2 \dots \bar{p}'_L G_L = r'_{\text{eq}}
\end{align*}
Therefore $\nu \circ \nu^{-1} = \mathds{1}_r$ and $\nu$ is bijective.

Lastly, we prove $\nu$ is equivariant. Let $(p'_1 G_1, p'_2 G_2, \dots, p'_L G_L) = u\cdot (p_1 G_1, p_2 G_2, \dots, p_L G_L)$ where the group action is implied by \Cref{eq:group_action}. From \Cref{eq:p_bar_inv_up}, we know that when $u\in G$, $\pi(\bar{p}_1^{\prime-1} u \bar{p}_1) \in G_1$. Recursively, we have
\begin{align} \label{eq:eq5}
    \bar{p}_l^{\prime-1}\dots \bar{p}_2^{\prime-1} \bar{p}_1^{\prime-1} u \bar{p}_1 \bar{p}_2 \dots \bar{p}_l \in G_l
\end{align}
for $l=1,\dots,L$. When $l=L$, from $\bar{p}_L^{\prime-1}\dots \bar{p}_2^{\prime-1} \bar{p}_1^{\prime-1} u \bar{p}_1 \bar{p}_2 \dots \bar{p}_l \in G_L$, we have
\begin{align*}
    \bar{p}'_1 \bar{p}'_2 \dots \bar{p}'_L G_L = u \cdot (\bar{p}_1 \bar{p}_2 \dots \bar{p}_L G_L)
\end{align*}
Hence $u\cdot \nu(p_1 G_1,p_2 G_2,\dots,p_L G_L) = \nu(u \cdot (p_1 G_1,p_2 G_2,\dots,p_L G_L))$. So that the group action given by \Cref{eq:group_action} is simplified to the left-action on the single group element.
\end{proof}

\subsection[]{Proof of Proposition~\ref{thm:inv-eqv-rep} (See page~\pageref{thm:inv-eqv-rep})}
\inveqvrep*
\begin{proof}
Let $f, f'\in \mathcal{I}_G$ be the input feature maps where $f'=\pi(g)f$. Let $[f_l,\;p_l G_l]$ and $[f'_l,\;p'_l G_l]$ be the feature maps and subsampling indexes output by the $l$-th subsampling layer for $f$ and $f'$ respectively. Let $[z_{\text{inv}},z_{\text{eq}}]=\operatorname{\textsc{enc}}(f)$ and $[z'_{\text{inv}},z'_{\text{eq}}]=\operatorname{\textsc{enc}}(f')$ and let $r_{\text{eq}}=\nu(z_{\text{eq}})$, $r'_{\text{eq}}=\nu(z'_{\text{eq}})$ where $\nu$ is given in \Cref{eq:nu}.

From \Cref{eq:group_action} and the equivariance of $\operatorname{\textsc{g-cnn}}_l(\cdot)$, we have
\begin{align*}
    f'_1 = \pi(\bar{p}_1^{\prime-1} g \bar{p}_1) f_1
\end{align*}
and recursively:
\begin{align} \label{eq:eq6}
    f'_l = \pi((\bar{p}'_1 \bar{p}'_2 \dots \bar{p}'_l)^{-1} g (\bar{p}_1 \bar{p}_2 \dots \bar{p}_l)) f_l
\end{align}
where $l=1,\dots,L$ and $(\bar{p}'_1 \bar{p}'_2 \dots \bar{p}'_l)^{-1} g (\bar{p}_1 \bar{p}_2 \dots \bar{p}_l) \in G_l$ (see \Cref{eq:eq5}).

Since $G_L=\{e\}$ when $l=L$, we have
\begin{align*}
    (\bar{p}'_1 \bar{p}'_2 \dots \bar{p}'_L)^{-1} g (\bar{p}_1 \bar{p}_2 \dots \bar{p}_L) = e
\end{align*}
Hence
\begin{align*}
    f'(e) &= f(e) \\
    (\bar{p}'_1 \bar{p}'_2 \dots \bar{p}'_L) &= g \cdot (\bar{p}_1 \bar{p}_2 \dots \bar{p}_L) 
\end{align*}
which can be rewritten as
\begin{align*}
    z'_{\text{inv}} &=z_{\text{inv}} \\
    r'_{\text{eq}} &= g \cdot r_{\text{eq}}
\end{align*}
From \Cref{thm:conversion} we have
\begin{align*}
    z'_{\text{eq}} = \nu^{-1}(r'_{\text{eq}}) = \nu^{-1}(u r_{\text{eq}}) = g \cdot \nu^{-1}(r_{\text{eq}}) = g \cdot z_{\text{eq}}
\end{align*}
Therefore,  for the encoders we have $[z_{\text{inv}}, g\cdot z_{\text{eq}}] = \operatorname{\textsc{enc}}(\pi(g) f)$.

For the decoders, let
\begin{align*}
    z'_{\text{eq}} &= (p'_1 G_1, p'_2 G_2, \dots, p'_L G_L) = g \cdot z_{\text{eq}}
\end{align*}

Since the feature map at the $l$-th subsampling layer is transformed according to \Cref{eq:eq6}, the sampling index is transformed accordingly:
\begin{align*}
    p'_l G_l = (\bar{p}_{l-1}^{\prime-1} \dots \bar{p}_2^{\prime-1} \bar{p}_1^{\prime-1} g \bar{p}_1 \bar{p}_2 \dots \bar{p}_{l-1}) p_l G_L
\end{align*}
fFrom the definition of equivariant upsampling in \Cref{eq:upsampling}, we have
\begin{align*}
    f'_{l-1} = \pi(\bar{p}_{l-1}^{\prime-1} \dots \bar{p}_2^{\prime-1} \bar{p}_1^{\prime-1} g \bar{p}_1 \bar{p}_2 \dots \bar{p}_{l-1}) f_{l-1}
\end{align*}
where $l=L,\dots,1$. When $l=1$, we have $f'_0 = \pi(g) f_0$ so that $\pi(g) \hat{f} = \operatorname{\textsc{dec}}(z_{\text{inv}}, g\cdot z_{\text{eq}})$.

\end{proof}

\section{Implementation Details} \label{sec:app_implementation}

In this section, we outline a few important implementation details, and leave other details to the reference code at \url{https://github.com/jinxu06/gsubsampling}.

\subsection{Data}

For Colored-dSprites and FashionMNIST datasets, we add colours to grayscale images from \cite{dsprites17} and \cite{xiao2017/online} by sampling  random scaling for each channel uniformly between $0.5$ and $1$ following \cite{pmlr-v97-locatello19a}. For FashionMNIST, we also apply zero-paddings to images to reach the size of $64\times 64$. We then translate the images with random displacements uniformly sampled from $\{(x,y)|-18\leq x,y \leq 18, x, y\in \mathbb{Z}\}$, and rotate the images with uniformly sampled degrees from $\{\frac{360\times k}{32}|k=0,\dots,32\}$.
We use the original Multi-dSprites dataset as provided in \cite{multiobjectdatasets19}. For CLEVR6, we crop images from the original CLEVR \citep{johnson2017clevr} at y-coordinates $(29, 221)$ bottom and top, and at x-coordinates $(64, 256)$ left and right as stated in \cite{burgess2019monet}. We then resize the images to $64\times 64$ so that we can use the same model for both multi-object datasets. We only use images with up to $6$ objects in CLEVR following \cite{greff2019multi}. For evaluation, we use a randomly sampled test set with $10000$ examples that has no overlap with training data for all datasets.
In \Cref{fig:datasets}, we show examples from different datasets.

\begin{figure}[t] 
     \begin{subfigure}[b]{0.5\textwidth}
         \centering
         \includegraphics[width=0.8\textwidth]{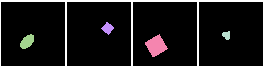}
    \caption{Colored-dSprites}
     \end{subfigure}
     \begin{subfigure}[b]{0.5\textwidth}
         \centering
         \includegraphics[width=0.8\textwidth]{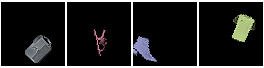}
     \caption{FashionMNIST}
     \end{subfigure}
     \begin{subfigure}[b]{0.5\textwidth}
         \centering
         \includegraphics[width=0.8\textwidth]{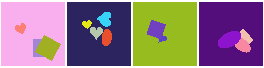}
     \caption{Multi-dSprites}
     \end{subfigure}
     \begin{subfigure}[b]{0.5\textwidth}
         \centering
         \includegraphics[width=0.8\textwidth]{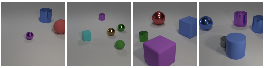}
     \caption{CLEVR6}
     \end{subfigure}
     \caption{}
     \label{fig:datasets}
\end{figure}

\subsection{Model Architectures} 

\paragraph{The equivariant map $\Phi$} One can insert arbitrary equivariant layers before and after we take the norm $\|\cdot\|_1$ in \Cref{eq:phi}. In experiments, we apply mean subtraction followed by average pooling with kernel size $5$ before taking the norm, and apply Gaussian blur with kernel size $15$ after taking the norm. They are inserted in the purpose of smoothing feature maps and avoiding non-unique argmax when possible (in \Cref{sec:app_phi} we discuss the case when non-unique argmax cannot be avoided). In practice, we use \Cref{eq:phi_1} to obtain all subsampling indexes at the same time rather than passing through $\Phi$ multiple times.

\paragraph{Autoencoders} 
For all single object experiments, we use $5$ layers of ($G$-)equivariant convolutional layers in encoders, and the decoders mirror the architecture of the encoders except for the output layers. In baseline models, we use strided convolution as a way to perform subsampling/upsampling, while in equivariant models we use equivariant downsampling/upsampling.
We rescale the number of channels such that the total number of parameters of the models roughly match one another. However, exact correspondence is not achievable because exact equivariant models use equivariant subsampling to transform feature maps into vectors at the final layer of the encoder, while baseline models apply flattening. Please see the reference implementation for other details about network architectures.

We use scale factor $2$ for all subsampling and upsampling layers in baseline models. For GAE-$p1$, the feature maps at each layer are defined on the following chain of nested subgroups:
$\mathbb{Z}^2 \geq (2\mathbb{Z})^2 \geq (4\mathbb{Z})^2  \geq (8\mathbb{Z})^2 \geq (16\mathbb{Z})^2 \geq \{e\}$. For For GAE-$p4$, we use $\mathbb{Z}^2 \rtimes \mathsf{C}_4 \geq (2\mathbb{Z})^2 \rtimes \mathsf{C}_4 \geq (4\mathbb{Z})^2 \rtimes \mathsf{C}_4 \geq (8\mathbb{Z})^2 \rtimes \mathsf{C}_4  \geq (16\mathbb{Z})^2 \rtimes \mathsf{C}_2 \geq \{e\}$. And for GAE-$p4m$, we use
$\mathbb{Z}^2 \rtimes (\mathsf{C}_4 \rtimes \mathsf{C}_2) \geq (2\mathbb{Z})^2 \rtimes (\mathsf{C}_4 \rtimes \mathsf{C}_2) \geq (4\mathbb{Z})^2 \rtimes (\mathsf{C}_4 \rtimes \mathsf{C}_2) \geq (8\mathbb{Z})^2 \rtimes (\mathsf{C}_4 \rtimes \mathsf{C}_2) \geq (16\mathbb{Z})^2 \rtimes (\mathsf{C}_2 \rtimes \mathsf{C}_2) \geq \{e\}$.

\paragraph{Object discovery} For baseline models, we adopt the exact same architecture as the original MONet \citep{burgess2019monet} using the implementation provided by \cite{engelcke2020genesis}. For MONet-GAEs, we simply replace Component VAEs in the original MONet with our V-GAEs. Both Component VAEs and V-GAEs have a latent size of $16$.

\subsection{Hyperparameters}

For all single object experiments, we use Adam optimizer \citep{DBLP:journals/corr/KingmaB14} with a learning rate of $0.0001$ and a batch size of $16$. We use $16$-bits precision to enable faster training and reduce memory consumption. For experiments on multi-object datasets, hyperparameters will match the original MONet \citep{burgess2019monet} except that we still use a batch size of $16$ instead of $64$ stated in the original paper. This is because we observed that in the low data regime, batch size $16$ trains faster and performs no worse than batch size $64$ for the problems we considered here.

\subsection{Computational Resources} \label{subsec:app_computes}

In theory, the only computational overhead is caused by computing sampling indices, which is negligible compared to the forward pass of ($G$-)Convolutional layers. In practice, our current implementation uses \texttt{torch.gather} to perform subsampling, and relies on for-loops over data batches when applying group actions to feature maps, which we believe can be made more efficient. Hence on a single GeForce GTX $1080$ GPU card, a standard GAE-p1 takes around $30$ minutes to train for $100k$ steps, compared to $16$ minutes for standard ConvAEs.

\end{document}